\documentclass[12pt]{article}
\usepackage[letterpaper,top=2cm,bottom=2cm,left=3cm,right=3cm,marginparwidth=1.75cm]{geometry}
\usepackage[english]{babel}
\usepackage[utf8]{inputenc}
\usepackage[T1]{fontenc}  
\usepackage{hyperref}      
\usepackage{url}      
\usepackage{booktabs}     
\usepackage{nicefrac}       
\usepackage{microtype}     
\usepackage{amsthm,amsmath,amsfonts,amssymb, mathrsfs, mathtools}
\usepackage{graphicx}
\usepackage{xcolor}
\usepackage{bm}
\usepackage{adjustbox}
\usepackage{comment}
\usepackage{bbm}
\usepackage{enumitem}
\usepackage[max6]{authblk}
\usepackage{caption}
\usepackage{subcaption}
\usepackage{array}
\usepackage{arydshln}
\usepackage{multirow}
\usepackage{multicol}
\usepackage{makecell}
\usepackage{natbib}

\usepackage{authblk}

\newtheorem{proposition}{Proposition}

\newtheorem{lemma}{Lemma}

\newtheorem{assumption}{Assumption}

\newcommand*{\eg}{\emph{e.g.}{}}
\newcommand*{\ie}{\emph{i.e.}{}}

\renewcommand{\d}{\mathrm{d}}

\DeclareMathOperator*{\argmax}{argmax}

\usepackage{algorithm2e}
\RestyleAlgo{ruled}

\title{\bf Recurrent Neural Goodness-of-Fit Test\\ for Time Series}
\date{\vspace{-5ex}}
\author[1]{Aoran Zhang} 
\author[1]{Wenbin Zhou}
\author[2]{Liyan Xie} 
\author[1]{Shixiang Zhu}
\affil[1]{Carnegie Mellon University}
\affil[2]{University of Minnesota}

\begin{document}
  \maketitle

\begin{abstract}
Time series data are crucial across diverse domains such as finance and healthcare, where accurate forecasting and decision-making rely on advanced modeling techniques. While generative models have shown great promise in capturing the intricate dynamics inherent in time series, evaluating their performance remains a major challenge. Traditional evaluation metrics fall short due to the temporal dependencies and potential high dimensionality of the features. In this paper, we propose the REcurrent NeurAL (RENAL) Goodness-of-Fit test, a novel and statistically rigorous framework for evaluating generative time series models. By leveraging recurrent neural networks, we transform the time series into conditionally independent data pairs, enabling the application of a chi-square-based goodness-of-fit test to the temporal dependencies within the data. This approach offers a robust, theoretically grounded solution for assessing the quality of generative models, particularly in settings with limited time sequences. We demonstrate the efficacy of our method across both synthetic and real-world datasets, outperforming existing methods in terms of reliability and accuracy. Our method fills a critical gap in the evaluation of time series generative models, offering a tool that is both practical and adaptable to high-stakes applications.
\end{abstract}

\section{INTRODUCTION}

Time series data are pervasive across various domains, playing a critical role in applications requiring temporal insights, such as forecasting stock prices, optimizing trading strategies in finance, healthcare monitoring, and managing power generation and demand in energy systems \citep{han2019review, zhu2021quantifying, wu2024counterfactual, zhang2024self}. 
As the complexity and volume of time series data continue to rise, the demand for sophisticated modeling techniques grows in parallel. 
In this context, generative time series models have emerged as powerful tools, enabling probabilistic time series prediction \citep{li2022generative, liu2023flow, albergo2023building}, the simulation of future scenarios \citep{dong2023conditional, sattarov2023findiff}, and the generation of synthetic data that captures complex temporal dependencies \citep{yoon2019time, zhu2021imitation, zhu2022neural, suh2023autodiff, suh2024timeautodiff}. 

\begin{figure}[!t]
    \centering
    \includegraphics[width=\linewidth]{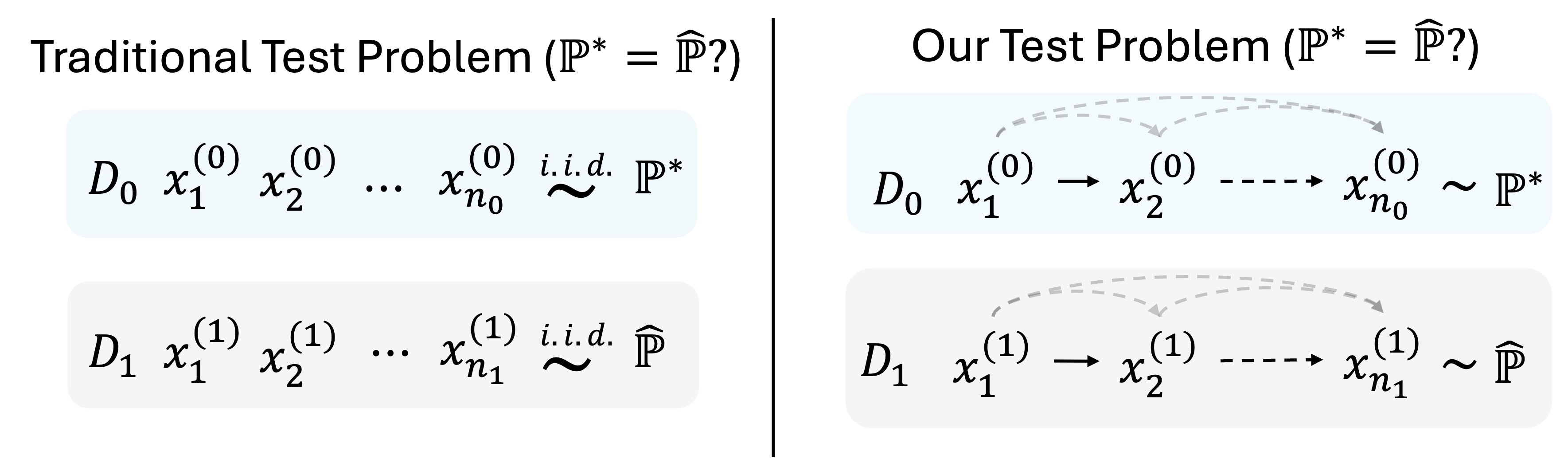}
    \caption{An illustration of our problem setup. In traditional test problems, both the real data ($D_0$) and the model-generated data ($D_1$) are assumed to be \textit{i.i.d.}, following the underlying distributions $\mathbb P^\star$ and $\widehat{\mathbb P}$, respectively. However, in our test problem, both $D_0$ and $D_1$ exhibit general temporal dependencies. 
    }
    \label{fig:our-task}
    \vspace{-.2in}
\end{figure}

However, despite their growing popularity, evaluating the performance of generative time series models remains a significant challenge.
In image and text domains, evaluation often relies on heuristic metrics, such as Fr\'{e}chet Inception Distance (FID) for images \citep{NIPS2017_8a1d6947} and BLEU score for text \citep{papineni2002bleu}, or even subjective visual inspection. 
These methods, however, are not readily applicable for time series data due to the intrinsic temporal dependencies between observations. As a result, the absence of established and rigorous evaluation methods for generative time series models poses a significant barrier to their further adoption, particularly in high-stakes domains where reliability is paramount. Current metrics either fail to account for the dynamic, evolving nature of time series or are built on assumptions that often do not hold in real-world settings. This underscores the need for robust, statistically sound evaluation techniques tailored to time series models.

Goodness-of-fit (GOF) tests have long been a cornerstone for evaluating the accuracy of statistical models \citep{d2017goodness}.
However, applying these tests to modern time series data presents significant challenges:
($i$) \emph{Complex temporal dependencies}: 
The inherent temporal structure of time series data often violates the strict stationarity assumption that underpins many traditional tests \citep{chen2003empirical}. 
In practice, we often only have access to a small number of sequences -- or even a single time series -- which does not provide enough samples to observe stationarity. 
A naive attempt to address this issue is to partition the time series into smaller segments, but this creates dependent subsequences, which restricts the effectiveness of standard tests that require a larger set of independent samples \cite{baum2023kernel}.
($ii$) \emph{High dimensionality}: 
Contemporary time series datasets often exhibit high dimensionality, complicating both estimation and hypothesis testing, which can result in diminished statistical power. 
These challenges underscore the necessity for new approaches that account for the inherent temporal structure and can operate effectively with a limited number of time sequences while providing statistically sound evaluations of model performance.

In this paper, we present a novel REcurrent NeurAL (RENAL) Goodness-of-Fit test tailored for general time series data, including both regularly and irregularly sampled time series. Our approach overcomes the challenges posed by temporal dependencies by transforming time series into conditionally independent data pairs using recurrent neural architectures such as Long Short-Term Memory (LSTM) and Transformers \citep{graves2012long, vaswani2017attention}.
We begin by leveraging recurrent networks to encode the time series into low-dimensional history embeddings, preserving temporal structure while mapping the original distribution into a low-dimensional Markov Chain. This enables us to decompose dependent observations into independent embedding pairs. 
We then propose a novel chi-square test to compare the empirical transition probabilities of history embeddings from real data and the generative model of interest. 
This method provides a theoretically grounded and practical tool for evaluating complex, potentially high-dimensional time series models. We demonstrate its strong empirical performance across various types of time series data, including both synthetic and real-world datasets, spanning traditional time series and discrete event sequences.

Our key contributions are summarized as follows.
\begin{enumerate}\setlength\itemsep{0em}
    \item We introduce the first goodness-of-fit test for general time series data, by transforming the problem into testing differences in transition probabilities of learned history embeddings. The resulting test offers an effective tool for evaluating the quality of generative time series models.
    \item We present the asymptotic distribution of test statistics under the null hypothesis, providing guarantees on the accuracies of the proposed test.
    \item We empirically validate our method, demonstrating its superior performance over state-of-the-art methods across both synthetic and real-world datasets, including both regularly and irregularly sampled time series.
\end{enumerate} 

\vspace{-0.1in}
\paragraph{Related Work} 

The evaluation of time-series generative models remains challenging due to the lack of a comprehensive and widely accepted framework.
Some measures have been proposed, including specialized metrics such as Fr\'{e}chet Inception Distance (FID) for images \citep{NIPS2017_8a1d6947} and BLEU score for text \citep{papineni2002bleu}, and more general distributional measures like Kullback-Leibier (KL) divergence \citep{moreno2003kullback, chan2005probabilistic}, Wasserstein distance \citep{xiao2017wasserstein}, and Maximum Mean Discrepancy (MMD) \citep{tpprl} have been adopted.
However, these measures are not inherently designed for goodness-of-fit (GOF) of time series data \citep{d2017goodness}.

Most GOF tests for time series are parametric. One common approach is to conduct residual-based tests that measure the discrepancy between residual distributions or assess residuals' autocorrelation \citep{ljung1978measure, box1970distribution, escanciano2006goodness, gallagher2015weighted}. Another approach is to conduct likelihood ratio-based tests to compare nested models  \citep{chernoff1954distribution, akaike1974new, burnham2004multimodel, chen2003empirical}. Stein-based methods assess model fit via Stein discrepancy, with applications in point processes \citep{chwialkowski2016kernel,yang2019stein}. These parametric methods typically impose strict assumptions about data distributions and are unsuitable for black-box generative models with unknown structures. 

In contrast, non-parametric GOF tests offer more flexibility without strict distributional assumptions. A popular approach measures the distance between kernel embeddings of distributions in a reproducing kernel Hilbert space using MMD \citep{gretton2012kernel, gretton2009fast,jitkrittum2016interpretable,chwialkowski2015fast, schrab2023mmd} or Fisher discriminant analysis \citep{eric2007testing}. The kernel can be enhanced with deep neural networks capture more cmoplex structures \citep{liu2020learning, kirchler2020two, zhu2023sequential}. Other methods use scores like  Hyv\"arinen scores \citep{wu2022score} or the Generalized Score \citep{wei2021goodness} to measure the discrepancies. 
However, most of these non-parametric tests overlook temporal dependencies and assume the data is {\it i.i.d.} or strictly stationary. While a non-parametric test based on Kernel Stein Discrepancy (KSD) exists for sequential models, it focuses on dependencies for neighboring events only and is designed for discrete-time series \citep{baum2023kernel}. We aim to design a GOF test to tackle continuous-time time series with long-term dependencies. 

Our work offers several key advantages over prior approaches. Due to its parametric specification, it is highly scalable compared to non-parametric methods when applied to large-scale data. Meanwhile, it can achieve high performance in complex and high-dimensional settings (\textit{e.g.}, generative models) compared to other traditional parametric methods, which can be credited to the strong modeling power of the recurrent neural networks that it incorporates.

\begin{figure*}[!t]
    \centering
    \includegraphics[width=\linewidth]{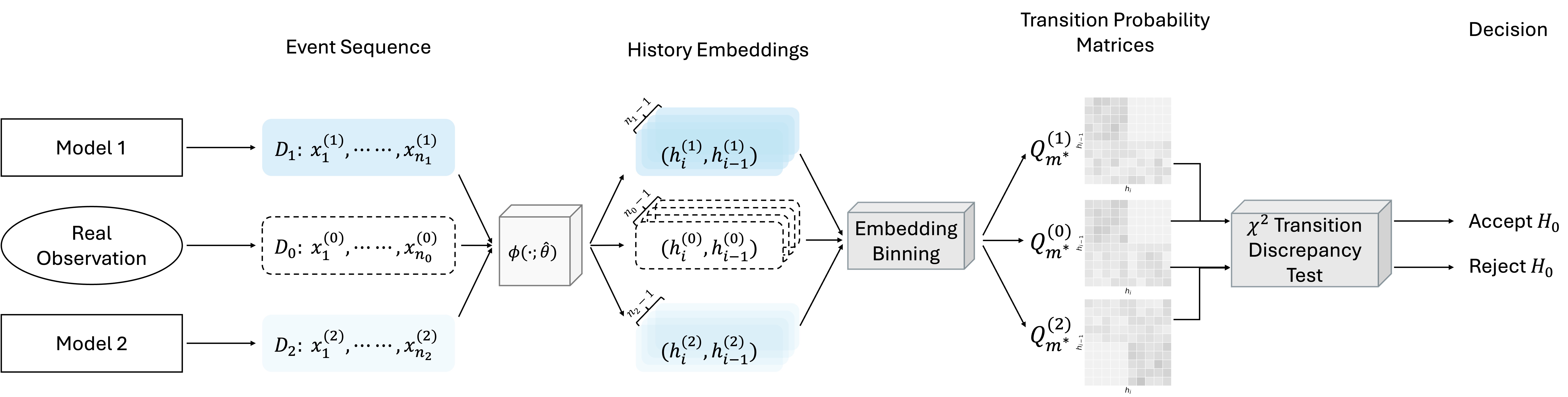}
    \caption{
    Architecture of the proposed framework. Real-world observations are compared to model-generated sequences, with darker blue indicating better fits. We first use a recurrent neural network $\phi$ to extract conditionally independent history embeddings. Then we construct their transition probability matrices using these embeddings and evaluate the fit with a chi-square discrepancy test.
    }
    \label{fig:architecture}
\end{figure*}

\section{RECURRENT NEURAL GOODNESS-OF-FIT}\label{sec:main}

In this section, we present a flexible framework for testing general time series data, termed REcurrent NeurAL (RENAL) Goodness-of-Fit. The overall architecture is summarized in Figure~\ref{fig:architecture}.

\subsection{Problem Setup}\label{sec:setup}

Consider two stationary time series datasets: $D_0 = \{ x_1^{(0)}, \ldots, x_{n_0}^{(0)} \}$, obtained from observations of a real-world process, and $D_1 = \{ x_1^{(1)}, \ldots, x_{n_1}^{(1)} \}$, synthetic data sequences generated by a learned time series model. 
Here, $n_0$ and $n_1$ are the lengths of the respective sequences, and each observation $x_i^{(k)}$ belongs to the sample space $\mathcal{X} \subseteq \mathbb{R}^d$, where $k \in \{0, 1\}$ indicates the dataset index and $d$ is the data dimension.
We assume that the data sequences are generated sequentially with strict temporal ordering. This means each observation $x_i^{(k)}$ depends solely on the preceding observations within its own sequence: $x_1^{(k)}, \ldots, x_{i-1}^{(k)}$.

Let $\mathbb{P}^*$ denote the true (but unknown) probability distribution governing the real-world time series $D_0$, and let $\widehat{\mathbb{P}}$ represent the probability distribution learned by the time series model, from which $D_1$ is sampled.
Our objective is to assess whether the learned time series model $\widehat{\mathbb{P}}$ accurately captures the underlying distribution of the real-world data. Formally, we aim to test the following hypotheses:
\begin{equation}
    H_0 : \mathbb{P}^* = \widehat{\mathbb{P}} \quad \text{versus} \quad H_1 : \mathbb{P}^* \neq \widehat{\mathbb{P}}.
    \label{eq:gof-objective}
\end{equation}
This constitutes a \emph{goodness-of-fit} (GOF) test for a time series model. In other words, we aim to test whether the real sequences $D_0$ are likely to be sampled from the tentative model $\widehat{\mathbb{P}}$. Note that $\widehat{\mathbb{P}}$ is also unknown except the sequences $D_1$ generated from it.
To assess the effectiveness of a GOF test, we use the Type-I and Type-II accuracy as the evaluation metric. We also choose the average of Type-I and Type-II accuracy as an overall metric\footnote{We may also use the weighted average of Type-I and Type-II accuracies, depending on which accuracy is prioritized.}.

Two key challenges arise in this task:
($i$) Each observation depends on all previous observations in its sequence, meaning the data are not independent. When we have a very limited number of sequences -- possibly even just one sequence for each distribution ($D_0$ and $D_1$) -- the lack of independence restricts the applicability of standard statistical tests.
($ii$) The sample space $\mathcal{X}$ may be high-dimensional, which complicates the estimation and comparison of the distributions $\mathbb{P}^*$ and $\widehat{\mathbb{P}}$.
The curse of dimensionality implies that the number of required samples increases exponentially with the dimensionality $d$, making distributional comparisons even more challenging.

\subsection{Recurrent Neural Representation}\label{sec:rec-neu-rep}

To address the challenges posed by temporal dependencies and high dimensionality in time series data, our main idea is to transform the original inter-dependent data sequences into a set of low-dimensional and conditionally independent data pairs. 
Rather than directly testing the raw time series data -- which is complicated by sequential dependencies -- we leverage recurrent neural architectures to create history embeddings that encapsulate temporal information.

We consider the time series modeled by Neural Ordinary Differential Equations (ODEs), a flexible and powerful framework capable of capturing continuous-time dynamics, making it well-suited for both regularly and irregularly sampled time series data  \citep{chen2018neural}.
Specifically, for a continuous-time series $\{x(t),t\geq 0\}$, we define a low-dimensional \emph{history embedding} $h(t) \in \mathcal{H} \subseteq \mathbb{R}^p$,
where $p \le d$ and $\mathcal{H}$ denotes the embedding space. This latent embedding captures the key information about the historical observations up to time $t$. The evolution of the embedding is governed by the differential equation:
\begin{equation}\label{eq:ode_h}
    \frac{\d h(t)}{\d t} = f(h(t),x(t)),
\end{equation}
where $f$ is the update function modeling the underlying dynamics over continuous time.
Here, $h(t)$ serves as a summarization of the time series $x(t)$ up to time $t$, effectively capturing temporal dependencies in a condensed form. 
Given the embedding, the observation $x(t)$ is a random variable with mean value $g(h(t))$, where $g: \mathcal{H} \mapsto \mathcal{X}$ is a decoding function.  

In practice, we work with $n$ discrete observations $\{t_i\}_{i=1}^n$ and adopt the Euler method to approximate the solution to the ODE model \eqref{eq:ode_h} \citep{chen2018neural}:
\begin{equation}\label{eq:h-discrete}
    h_{i+1}=h_i+f(h_i,x_i)\Delta t_i,
\end{equation}
where $h_i$ and $h_{i+1}$ are the history embeddings at times $t_i$ and $t_{i+1}$, respectively, and $\Delta t_i=t_{i+1}-t_i$ is time step between them. 

Inspired by the model \eqref{eq:h-discrete}, we parameterize the history embedding updates as 
\begin{equation}\label{eq:h-embedding}
 h_{i+1} = \phi(x_i, h_i; \theta),   
\end{equation}
using an embedding function $\phi(\cdot,\cdot;\theta): \mathcal{X} \times \mathcal{H} \mapsto \mathcal{H}$, which is modeled as a neural network with $\theta$ representing the network's weights.  The embedding function updates the history embedding $h_{i+1}$ based on the current observation $x_i$ and the previous history $h_{i}$. 

\noindent\emph{Learning}. The embedding network $\phi(\cdot,\cdot;\theta)$, together with the decoding network $g(\cdot;\varphi)$ can be learned by minimizing the one-step predictive loss. Specifically, we let $h_{i+1} = \phi(x_i, h_i; \theta)$ as in \eqref{eq:h-embedding} and $\widehat{x}_{i+1}=g(h_{i+1};\varphi)$, and the model parameters $\theta$ and $\varphi$ are jointly learned by minimizing a predictive loss, \eg, the mean square error $\sum_{i} \left\Vert\widehat{x}_i - x_i\right\Vert^2$. We denote the learned network as $\phi(\cdot,\cdot;\widehat{\theta})$ and $g(\cdot;\widehat{\varphi})$.

\subsection{Reformulation of GOF}

To facilitate the application of the GOF test in \eqref{eq:gof-objective}, we impose the following assumption on the learned embedding function.

\begin{assumption}[Consistency]
    \label{assump:sufficient}
    The learned embedding function $\phi(\cdot,\cdot;\widehat{\theta})$ is consistent, meaning it approximates the true underlying embedding function in \eqref{eq:h-discrete} as closely as necessary. Therefore, the learned embedding $h_{i+1}$ captures all relevant information from $x_i$ and $h_i$.
\end{assumption}
This assumption ensures that the history embeddings produced by our model are sufficient representations of the observed history, effectively summarizing all relevant past information required for accurate time series modeling.

Prior research has demonstrated the empirical validity of this assumption across various recurrent neural architectures, including RNNs, LSTMs, and Transformers \citep{mei2017neural, chen2018neural, zuo2020transformer, dong2023conditional}. Furthermore, even in the case when this assumption does not hold, our numerical experiments indicate that our method remains robust in performance. 
This resilience stems from the binary nature of the GOF test, where a clear distinction between null and alternative sequences is typically sufficient to trigger an alert.

Under Assumption~\ref{assump:sufficient}, we show in the following Lemma that the sequence of history embeddings $\{ h_i\}_{i=1}^n$ possesses the \emph{Markov} property and is \emph{homogeneous} over time.
\begin{lemma}
    \label{lemma:markov}
    The history embedding sequence $\{h_i\}_{i=1}^n$ is a homogeneous Markov chain, \ie, for any set $B \subset \mathcal{H}$:
    \[
        \mathbb{P}\{h_i \in B | h_{i-1}, h_{i-2}, \ldots, h_1\} = \mathbb{P}\{h_i \in B | h_{i-1}\},
    \]
    and 
    \[
        \mathbb{P}\{h_{i+1} \in B| h_{i}\} = \mathbb{P}\{h_{i} \in B | h_{i-1}\}.
    \]
    The proof is provided in Appendix \ref{app:proof-markov}.
\end{lemma}
Therefore, the distributional behavior of the learned history embeddings can be fully characterized by their one-step {\it transition} density function $Q:\mathcal H \times \mathcal H \mapsto \mathbb R_{\geq0}$. Specifically, for any subset $B \subseteq \mathcal{H}$ and for all $i$, we have
\[
\mathbb P\{h_{i+1} \in B |h_i = h\} = \int_{h' \in B} Q(h,h')dh',
\]
where $Q(h, \cdot)$ serves as the conditional probability density function of $h_{i+1}$ given $h_i = h$. 

\begin{figure}[!t]
    \centering
    \begin{subfigure}[t]{0.30\textwidth}
        \centering
        \includegraphics[width=\linewidth]{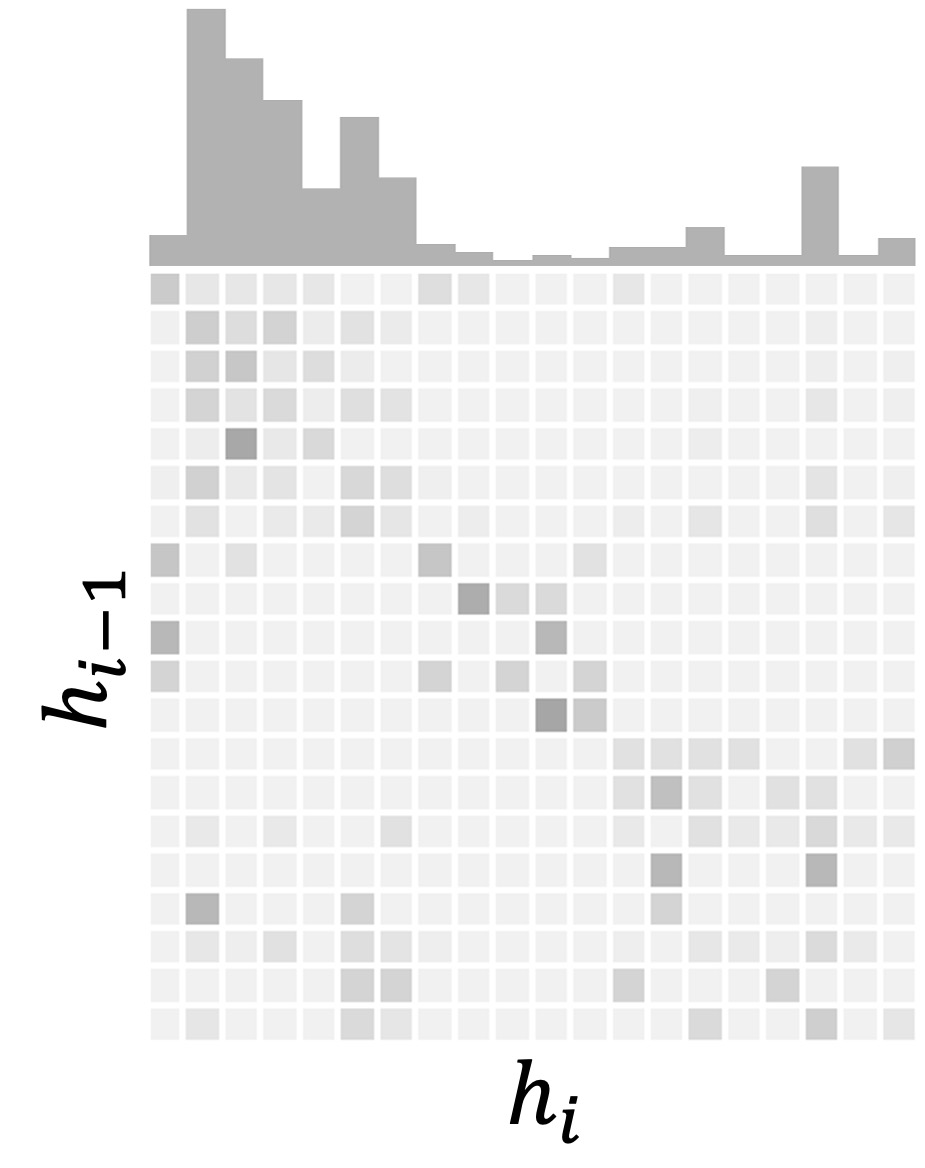}
        \caption{Real $Q^\star$}
        \label{fig:h1}
    \end{subfigure}
    \begin{subfigure}[t]{0.30\textwidth}
        \centering
        \includegraphics[width=\linewidth]{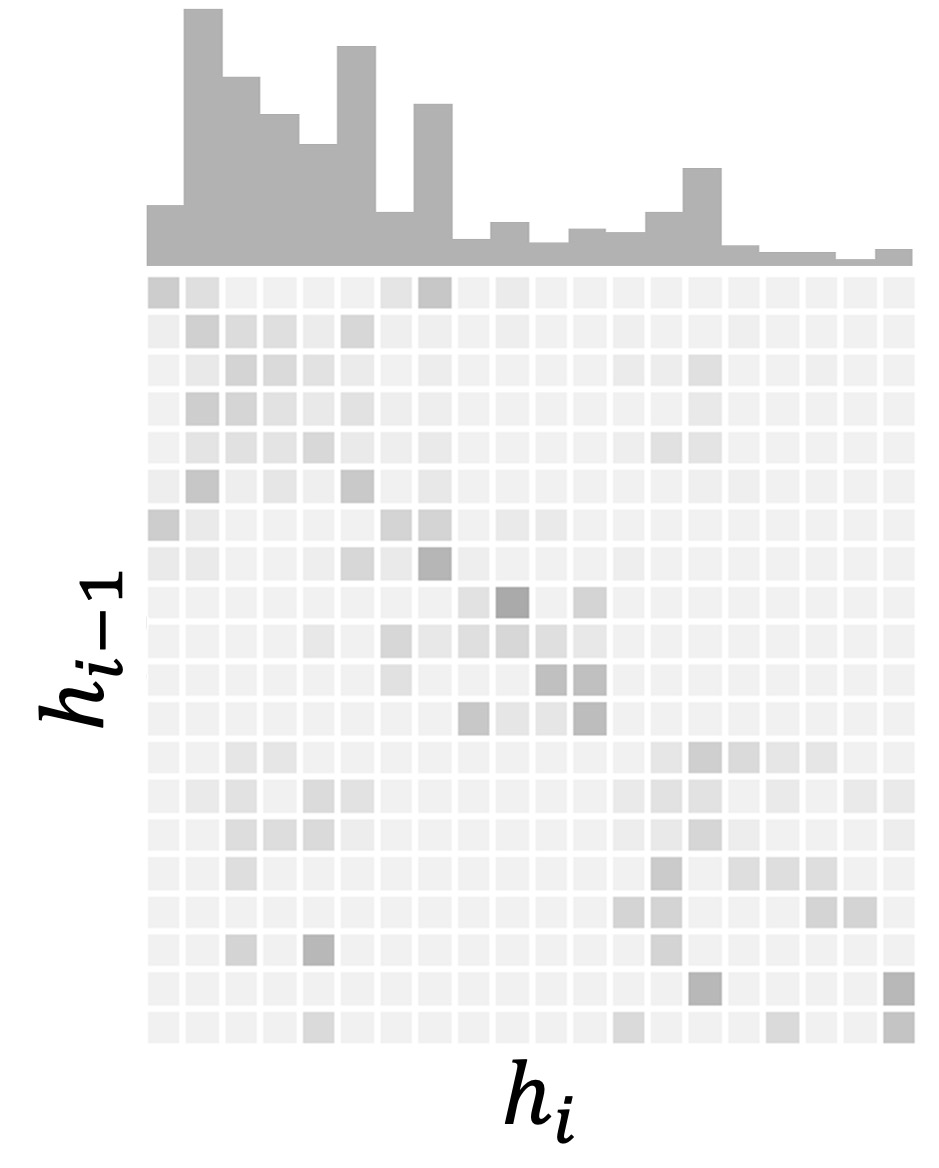}
        \caption{$\widehat{Q}_1$ ($0.99$)}
        \label{fig:h2}
    \end{subfigure}
    \begin{subfigure}[t]{0.30\textwidth}
        \centering
        \includegraphics[width=\linewidth]{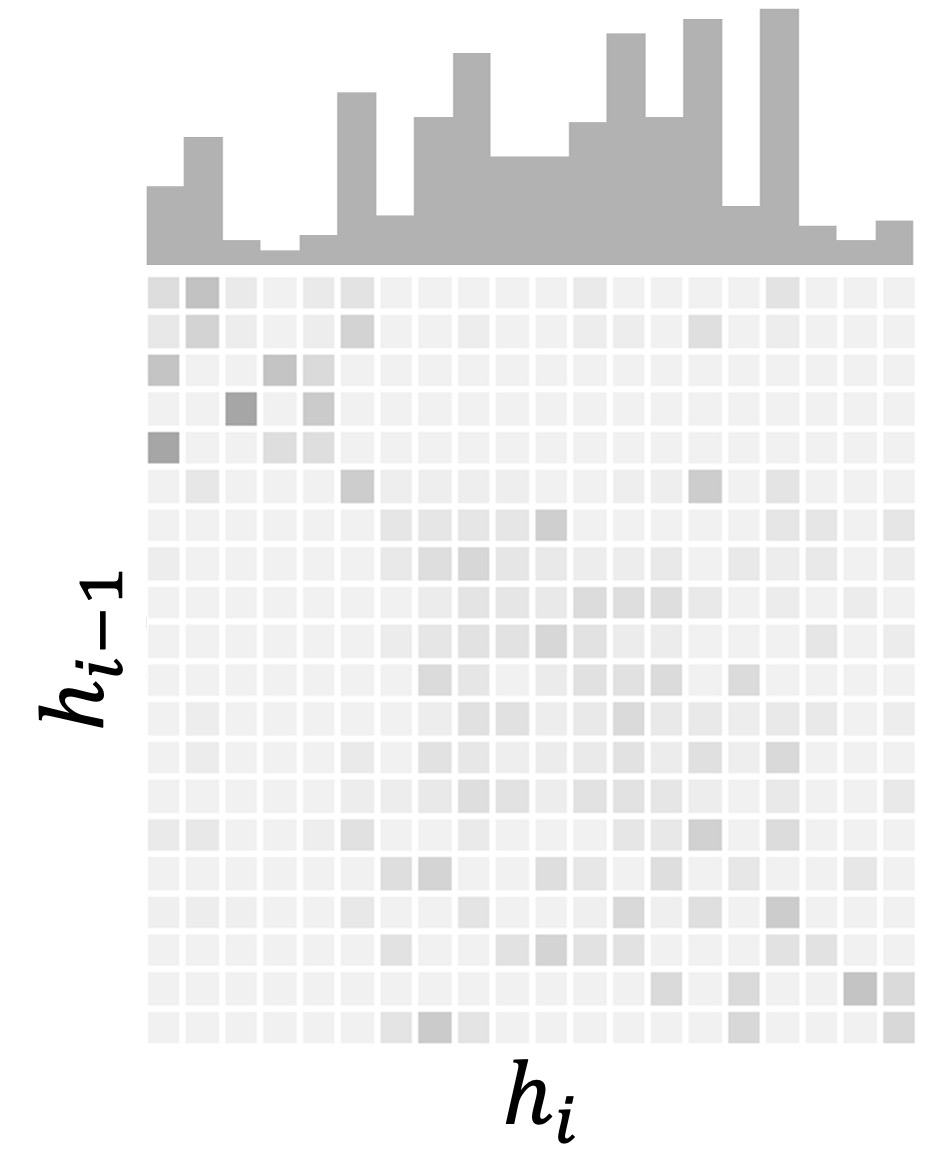}
        \caption{$\widehat{Q}_2$ ($\approx 0$)}
        \label{fig:h3}
    \end{subfigure}
    \caption{
    Transition probability matrices of history embeddings $Q$ from (a) real data, (b) data generated by Model $1$, and (c) data generated by Model $2$. Model $1$ exhibits a better fit compared to Model $2$, as evidenced by the closer resemblance between the histograms in (a) and (b). The number in the parentheses indicates the corresponding testing score. 
    }
    \label{fig:his-embs}
\end{figure}

This formulation allows us to pivot our original GOF test in \eqref{eq:gof-objective} to test the equality of the one-step transition probabilities for history embeddings. By leveraging the Markov property, we simplify the problem as shown in the following proposition.
\begin{proposition}
    \label{prop:markov-transition}
    The original GOF test in \eqref{eq:gof-objective} can be reformulated as:
    \begin{equation}
        H_0 : {Q}^\star = \widehat{Q} \quad \text{versus} \quad H_1 : {Q}^\star \neq \widehat{Q},
        \label{eq:markov-transition-test-objective}
    \end{equation}
    where $Q^\star$ denotes the true transition density function of the history embeddings derived from the real data, and $\widehat{Q}$ denotes the transition density function derived from the model-generated data.
\end{proposition}

Figure~\ref{fig:his-embs} provides an illustrative comparison of transition probability matrices estimated from three different datasets. Panel (a) displays the matrix derived from the real data used for training, which is a sequence generated by a self-exciting point process (SE). Panels (b) and (c) show the transition probability matrices estimated from data generated by two fitted models, denoted as Model 1 and Model 2, respectively. We set Model 1 to be the same as SE, while Model 2 represents a self-correcting point process (SC), which differs significantly from SE.  
Observing these matrices, we can see that Model $1$ in panel (b) offers a better fit compared to Model $2$ in panel (c), as the pattern of its matrix more closely resembles that of the real data shown in panel (a). 

\subsection{Proposed Algorithm} \label{sec:algo}

To operationalize the reformulated GOF test in \eqref{eq:markov-transition-test-objective}, we proceed by partitioning the embedding space and calculating both transition probability matrices $Q^\star$ for the real data and $\widehat{Q}$ for the model-generated data.
We then perform a chi-square test using these matrices to assess the fit of the model. The complete testing procedure is outlined in Algorithm~\ref{alg:algo-box}.

We begin by learning the embedding function $\phi$ using the real-world observations $D_0$.
After training, we apply the learned embedding function to both datasets $D_0$ and $D_1$ and obtain their embedding sequences $\{h_i^{(0)}\}$ and $\{h_i^{(1)}\}$. 

\paragraph{Step 1: Embedding Binning}
We partition the embedding space $\mathcal{H}$ into $m$ bins of equal size, denoted as $\mathcal{H}_1,\ldots,\mathcal{H}_m$. 
Such discretization allows us to approximate the continuous transition density function $Q$ with a discrete transition matrix. For a given embedding sequence $\{h_i\}_{i=1}^n$, we assign each embedding to a bin, resulting in a sequence of bin indices $\{y_i\}_{i=1}^n$, where 
$
    y_i = \sum_{u = 1}^m u \cdot \mathbbm{1}\left\{ h_i \in \mathcal{H}_u \right\},
$
and $\mathbbm{1}$ denotes the indicator function, which equals $1$ if the condition is true and $0$ otherwise.
We then compute the transition count matrix $C_m = (c_{uv})_{u,v=1}^m$, where each entry represents the number of transitions from bin $u$ to bin $v$ in the embedding sequence, for any $1\leq u,n\leq m$,
\begin{equation}\label{eq:c_uv}
    c_{uv} = \sum_{i=1}^{n-1} \mathbbm{1}\left\{ h_{i} \in \mathcal{H}_u \text{ and } h_{i+1} \in \mathcal{H}_v \right\}.
\end{equation}
We then construct the empirical transition probability matrix $Q_m \coloneqq (q_{uv})_{u,v=1}^m$ by normalizing each row:
\begin{equation}\label{eq:q_uv}
    q_{uv} = \frac{c_{uv}}{\sum_{v'=1}^m c_{uv'}}. 
\end{equation}
This matrix represents the estimated probabilities of transitioning from bin $u$ to bin $v$.

In Lemma~\ref{lem:discret-markov}, we establish that this binning method effectively transforms the original continuous-state Markov chain into a discrete-state Markov chain. The proof is provided in Appendix \ref{app:discrete-markov-proof}.
\begin{lemma}\label{lem:discret-markov}
    For a given continuous-state homogeneous Markov chain $\{h_i\}_{i=1,\dots,n}$, the corresponding discretized sequence $\{y_i\}_{i=1,\dots,n}$ forms a homogeneous Markov chain in a discrete state space.
\end{lemma}
We denote the empirical transition count and probability matrices obtained from $D_0$ as $C_m^{(0)}$ and $Q_m^\star$, following \eqref{eq:c_uv}-\eqref{eq:q_uv} , and the corresponding matrices from $D_1$ as $C_m^{(1)}$ and $\widehat{Q}_m$. By comparing the transition probability matrices $Q_m^\star$ and $\widehat{Q}_m$, we can perform the chi-square test to determine if there is a significant difference between the two distributions \citep{billingsley1961statistical}.

\noindent\emph{Optimal Bin Number Selection}.
Selecting an appropriate number of bins $m$ is crucial for effectively capturing the distributional differences between the two data sequences. An inadequate choice of $m$ can either oversimplify the data (if $m$ is too small) or introduce noise and sparsity issues (if $m$ is too large), adversely affecting the power of the statistical test.

To determine the optimal number of bins $m$, we formulate the following optimization problem, inspired by \cite{zhu1998second}:
\[
    \max_{m \ge 1}~\left \{ \lVert Q_m^{\star} - \widehat{Q}_m \rVert_F + \lambda\left( S(Q_m^{\star}) +  S(\widehat{Q}_m) \right )\right\},
\]
where $\lVert \cdot \rVert_F$ denotes the Frobenius norm that measures the discrepancy between two transition matrices, $\lambda$ is a user-defined parameter that control the trade-off between the discrepancy and the smoothness (or roughness) of the matrices, and $S(\cdot)$ is the smoothness measure of the transition matrix and is defined as:
\[
    S(Q_m) = -\sqrt{\sum_{u=2}^{m-1} \sum_{v=2}^{m-1} (\nabla^2 q_{uv})^2 }.
\]
Here $\nabla^2 q_{uv}$ denotes the second derivatives of $q_{uv}^{(k)}$, and in our binning scenario it is calculated as $\nabla^2 q_{uv} = q_{u+1,v} + q_{u-1,v} + q_{u,v+1} + q_{u,v-1} - 4q_{uv}$.

The rationale behind this optimization is to select the optimal number of bins $m$ by balancing two objectives: maximizing the discrepancy between the transition matrices of the real and generated data (to enhance the test's sensitivity to differences), and promoting smoothness in these matrices (to prevent noise and sparsity), thereby achieving a trade-off that leads to strong testing power. By optimizing this balance, we ensure the transition matrices capture meaningful differences without being affected by overfitting or random fluctuations.
We validate this procedure empirically and find that it enables the transition matrices to capture enough information about the differences in the trajectories, ultimately leading to strong testing power. The details of this procedure are provided in Appendix~\ref{app:smoothing}.

\paragraph{Step 2: $\chi^2$ Transition Discrepancy Test}
To assess the differences between the transition probability matrices $Q_m^\star$ and $\widehat{Q}_m$, we adopt the chi-square statistics, denoted by $W_m$.
The chi-square test is well-suited for measuring discrepancies between categorical distributions and, in our context, directly quantifies the differences between the transition matrices. We also establish its asymptotic distribution under temporal regimes following \cite{billingsley1961statistical} in Proposition~\ref{prop:chi},
with the proof provided in Appendix~\ref{app:chi-proof}. 

\begin{proposition}

\label{prop:chi} 
Given the transition count matrices $C_m^{(k)}$ and the corresponding $Q_m^\star=(q_{uv}^{\star})$ and $\widehat{Q}_m=(\widehat{q}_{uv})$ from \eqref{eq:q_uv}, and total event counts $n_{k}$ for $k \in \{0, 1\}$,
a chi-square test statistics $W_m$ is given by
\begin{equation}
  W_m = \sum_{u = 1}^m \sum_{v=1}^m \dfrac{c^{(0)}_{u}c^{(1)}_{u}}{c^{(0)}_{uv}+c^{(1)}_{uv}}\left( q_{uv}^{\star} - \widehat{q}_{uv}\right)^{2}
\label{eq:chi-statistic}
\end{equation}
where $c_u^{(k)} = \sum_{v=1}^m c_{uv}^{(k)}$ is the the total count of transitions from state $u$.
Under the null hypothesis $H_0 :  {Q}^\star = \widehat{Q}$, as $n_0, n_1 \to \infty$, the test statistics $W_m$ asymptotically follows a chi-square distribution with $m (m-1)$ degrees of freedom: 
\[
    W_m \sim \chi^2_{m(m-1)}.
\]
\end{proposition}

Given the fitted chi-squared statistics $W_m$, we test the reformulated null hypothesis in \eqref{eq:markov-transition-test-objective} by setting the rejection region of significance level $\alpha$ as:
\begin{equation}\label{eq:reject-region}
R_m(\alpha) = \left\{ W_m \ge \chi^2_{m(m-1)}(\alpha) \right\},    
\end{equation}
where $\chi^2_{m(m-1)}(\alpha)$ denotes the $(1-\alpha)$-quantile of the chi-square distribution with $m(m-1)$ degrees of freedom, and $\alpha$ is the pre-specified significance level.
If the null hypothesis is rejected, we conclude that the sequences $D_1$ are not generated from the distribution of $D_0$; otherwise, we infer the opposite. 
As a direct consequence of Proposition~\ref{prop:chi}, 
the type-I error rate (the probability of incorrectly rejecting $H_0$ when it is true) associated with this rejection region is asymptotically bounded by $\alpha$.
This underscores the effectiveness of the proposed testing procedure in controlling false positives under the asymptotic regime.

\begin{algorithm}[!t]
\small
\caption{RENAL Goodness-of-Fit Test}\label{alg:algo-box}
\SetKwInOut{Input}{Input}\SetKwInOut{Output}{Output}
\Input{$D_0 = \{x_1^{(0)}, \dots, x_{n_0}^{(0)} \}$, $D_1 = \{x_1^{(1)}, \dots, x_{n_1}^{(1)}\}$, smoothing parameter $\lambda$, set of bin numbers $M$, significance level $\alpha$.}
\textbf{Initialization:} $h_0^{(0)}$, $h_0^{(1)}$, $m^* \leftarrow 2$, $\ell^\star \leftarrow -\infty$\;
    $\widehat{\theta} \leftarrow $
    Learn model $\phi(\cdot;\theta)$ using $D_0$\;
\For{$k \in \{0, 1\}$}{
        \For{$i = 1$ \KwTo $n_k$}{
            $h_i^{(k)} = \phi(h_{i-1}^{(k)}, x_i^{(k)}; \widehat{\theta})$\;
        }
    }
$m^\star \leftarrow \argmax_{m^\prime}\{\lVert Q_{m^\prime}^\star - \widehat{Q}_{m^\prime}\rVert + \lambda \left(S(Q_{m^\prime}^\star) + S(\widehat{Q}_{m^\prime})\right)\}$\;
where $Q_{m}^{\star}$ and $\widehat{Q}_m$ are computed using $C_m^{(0)}$ and $C_m^{(1)}$, 
according to \eqref{eq:c_uv} and~\eqref{eq:q_uv}\;
\vspace{1ex}
Compute test statistics $W_{m^*}$ using \eqref{eq:chi-statistic}\;
\SetAlgoNoEnd
\If{$W_{m^*} > \chi^2_{m^*(m^*-1)}(\alpha)$}{
\vspace{.5ex}
\Return Reject\;
} 
\Else{
\vspace{.5ex}
\Return Accept\;
}
\end{algorithm}
    
We highlight RENAL's strong computational scalability in large-scale data settings due to the following factors:
($i$) Several acceleration techniques can enhance RENAL’s scalability, such as parallel processing with attention mechanisms \citep{bahdanau2014neural} and leveraging pre-trained models \citep{franceschi2019unsupervised} to reduce the number of training iterations required for convergence.
($ii$) RENAL extracts a low-dimensional representation of temporal dependencies using a compact transition matrix, which minimizes the computational overhead of the binning optimization procedure, even with large training datasets.
($iii$) As a parametric GOF testing method, RENAL’s inference cost scales with the number of model parameters rather than the number of training samples, making it more efficient than nonparametric methods in large-scale applications.

\section{EXPERIMENTS}\label{sec:exp}

In this section, we validate our proposed goodness-of-fit test against competing baselines (Section~\ref{exp:baseline}) across a broad spectrum of sequential data paradigms. Synthetic and real data studies are shown in Section~\ref{exp:syn data} and Section~\ref{exp:real-data} respectively.

\subsection{Experiment Setup}

We first introduce the baselines, evaluation metrics, and choices of neural network models used in the experiments. More detailed baselines and specific configurations are provided in Appendix~\ref{exp:baseline} and~\ref{exp:nn-details}.

\vspace{-0.1in}
\paragraph{Baselines}
We consider a broad range of existing GOF tests including:
($i$) Empirical Likelihood (\texttt{EL}): a test using kernel smoothing to assess the empirical likelihood for an $\alpha$-mixing process \citep{chen2003empirical},
($ii$) Portmanteau test with Weighted McLeod–Li statistics (\texttt{PT-$Q_w$}): a portmanteau test using the trace of squared autocorrelation matrices for stable statistics \citep{gallagher2015weighted},
($iii$) Sieve-based Cramer-von Mises (\texttt{S-CvM}): a sieve-bootstrap based test to approximate $p$-values of the Cramer-von Mises test  \citep{escanciano2006goodness},
($iv$) Maximum Mean Discrepancy (\texttt{MMD}): a kernel-based test measuring MMD distance  \citep{gretton2012kernel},
($v$) Stein-based Bootstrap test (\texttt{Stein}): a test using a Stein operator value along with a wild bootstrap \citep{chwialkowski2016kernel},
($vi$) Kernel Stein Discrepancy (\texttt{KSD}): a kernel Stein test specifically designed for general point processes \citep{yang2019stein}.
We also consider other baselines using recurrent neural representations with different partitions with details available in Appendix~\ref{exp:baseline}. 

\vspace{-0.1in}
\paragraph{Evaluation Metrics}
\label{exp:eval metric}

We use standard Type-I and Type-II accuracies to assess the performance of the proposed method. Type-I accuracy refers to the probability of correctly accepting the null hypothesis when it is true, while Type-II accuracy reflects the probability of correctly rejecting the null when it is false. To compare overall performance across different scenarios, we also compute the ``Average'' accuracies in the Tables below, by simulating multiple data sequence pairs from both $H_0$ and $H_1$. The average accuracy is defined as the proportion of correct decisions, i.e., accepting $H_0$ when true and rejecting it when false.  

\vspace{-0.1in}
\paragraph{Model Configuration}  \label{exp:setup}
For our method, we use LSTM \citep{graves2012long} as the embedding function $\phi$ for regularly sampled time series data, and Neural Hawkes \citep{mei2017neural} for irregularly sampled time series data, such as temporal point processes.
The neural network configurations and discretization settings are described in Appendix \ref{exp:nn-details}. All baseline methods are applied following their original implementations, with set-up details available in Appendix~\ref{exp:baseline}.
The significance level for all chi-square tests was set at $\alpha = 0.05$, which is used to determine the rejection region as in \eqref{eq:reject-region}.

\begin{table*}[ht]
\centering
\caption{Testing accuracy on synthetic dataset with $\alpha=0.05$. ``--'' indicates the method is not applicable.
} 
\vspace{-0.1in}
\label{tab:syn_acc}
\resizebox{\linewidth}{!}{
\begin{tabular}{c c c c c c c c @{\hspace{0.5cm}} c c c c c @{\hspace{0.5cm}} c c c c c}
\Xhline{2pt}
\multicolumn{2}{c}{\textbf{Methods}} &\multicolumn{6}{c }{\textbf{Time Series}} &  \multicolumn{5}{c }{\textbf{TPP}} & \multicolumn{5}{c}{\textbf{STPP}} \\ \cline{3-8}\cline{9-13}\cline{14-18}
\multirow{2}{*}{\textbf{\quad}}&$P^\star$ 
& \multirow{2}{*}{Average} & ARMA$(2,1)$ & ARMA$(2,1)$ & ARMA$(2,2)$ & ARMA$(2,1)$ & ARMA$(2,2)$  
&\multirow{2}{*}{Average}& SE & SC & SE & SC  
&\multirow{2}{*}{Average} &  STD & GAU & STD & GAU  \\
&$\widehat{P}$ 
 &    & ARMA$(2,1)$ & ARMA$(2,2)$ & ARMA$(2,2)$ & GARCH$(1,1)$ & GARCH$(1,1)$
&   & SE & SC & SC & SE 
&   & STD & GAU & GAU & STD  \\ \hline
\multicolumn{2}{c}{\texttt{EL}} & $0.37$ & $0.08$ & $0.33$ & $0.23$ & $0.76$ & $0.64 $
& -- & -- & -- & -- & -- & -- & -- & -- & --\\
\multicolumn{2}{c}{\texttt{PT}-$Q_W$} & $0.52$ & $0.91$ & $0.15$ & $0.93$ & $0.95$ & $0.08 $
& -- & -- & -- & -- & -- & -- & -- & -- & --\\
\multicolumn{2}{c}{\texttt{S-CvM}} & $0.44 $& $0.98$ & $0.08$ & $0.97$ & $0.02 $& $0.04 $
& -- & -- & -- & -- & -- & -- & -- & -- & -- \\
\multicolumn{2}{c}{\texttt{Stein}} & -- & -- & -- & -- & -- &-
& $0.51$ & $0.32$ & $0.47$ & $0.51$ & $0.74  $
& -- & -- & -- & -- \\
\multicolumn{2}{c}{\texttt{KSD}} & -- & -- & -- & -- & -- &-
&$0.56$  & $0.66$ &$ 0.78$ & $0.72$ & $0.08$ 
& -- & -- & -- & -- \\ 
\multicolumn{2}{c}{\texttt{MMD}} & $0.62$ & $0.82 $& $0.15$ & $0.72$ & $0.75 $& $0.82$ 
 & $0.45$ & $0.53$ & $0.61$ &$ 0.31$ & $0.35$
 & $0.43$ & $0.65$ & $0.68$ & $0.10$ & $0.26$ \\ 
\multicolumn{2}{c}{\texttt{EWD-2}} & $0.55$ & $0.70$ & $0.10$ & $0.68$ & $0.70$ & $0.84 $
& $0.53$ & $0.80$ & $0.69$ & $0.25$ & $0.37$ 
& $0.44$ & $0.75$ & $0.80$ & $0.07$ & $0.13$ \\
\multicolumn{2}{c}{\texttt{EWD-4}} & $0.61$ & $0.45$ & $0.65$ & $0.37$ & $0.98$ & $0.87$
 & $0.53$ & $0.46$ & $0.33$ & $0.63 $& $0.68$
& $0.48$ & $0.13$ & $0.09$ & $0.93$ & $0.75$ \\
\multicolumn{2}{c}{\texttt{EWD-6}} & $0.66$ & $0.60$ & $0.52$ & $0.58$ & $0.95$ & $0.91 $
& $0.56$ & $0.54$ & $0.63$ & $0.54$ & $0.52$ 
& $0.38$ & $0.68$ & $0.63$ & $0.07$ & $0.13$  \\
\multicolumn{2}{c}{\texttt{EWD-8}} &$0.66 $ & $0.80$ & $0.15$ & $0.76$ & $0.94$& $0.84$ 
& $0.53$ & $0.65$ & $0.76$ & $0.38$ & $0.39$ 
& $0.48$ & $0.83$ & $0.88$ & $0.12$ & $0.17$ \\
\multicolumn{2}{c}{\texttt{EWD-10}} & $0.67$ & $0.95$ & $0.15$ & $0.85$ & $0.77$ & $0.65 $
 & $0.52$ & $0.80$ & $0.80$ & $0.22$ & $0.29$
 & $0.46$ & $0.94$ & $0.85$ & $0.05$ & $0.09$\\ 
\multicolumn{2}{c}{\texttt{Scott}} & $0.56$  & $0.97$ & $0.08$ & $0.47$ & $0.72$ & $0.58 $
&$ 0.43$ & $0.99$ & $1$ & $0.03$ & $0.01$ 
& $0.49$ & $0.98$ & $1$ & $0.005$ & $0.001$ \\ 
\multicolumn{2}{c}{\textbf{\texttt{RENAL}}} & \textbf{0.72} & $0.71$ & $0.60$ & $0.65$ & $0.92$ & $0.95 $
 & \textbf{0.61} & $0.64$ & $0.62$ & $0.58$ & $0.56$
& \textbf{0.60} & $0.70$ & $0.57$ & $0.67$ &$ 0.44$ \\ \Xhline{2pt}
\end{tabular}}
\end{table*}

\begin{table*}[ht]
\centering
\caption{Testing accuracy on real data with $\alpha=0.05$. 
For earthquake data, we consider two cases: TPP (time only) and STPP (spatio-temporal).
}
\label{tab:earthquake}
\vspace{-0.1in}
\resizebox{\linewidth}{!}{
\begin{tabular}{cc  c c c c c @{\hspace{0.5cm}} c c c c c @{\hspace{0.5cm}} c c c c c}
\Xhline{2pt}
\multicolumn{2}{c}{Model}  & \multicolumn{5}{c}{Weather Time Series} &  \multicolumn{5}{c}{Earthquake TPP} & \multicolumn{5}{c}{Earthquake STPP} \\ \cline{3-17}
\multicolumn{1}{c}{\quad\quad} & $P^\star$ & \multicolumn{1}{c}{\multirow{2}{*}{Average}} & PIT & SFO & SFO & PIT 
& \multicolumn{1}{c}{\multirow{2}{*}{Average}} & JP & NC & NC & JP 
&\multicolumn{1}{c}{\multirow{2}{*}{Average}} & JP & JP & NC & NC  \\
\multicolumn{1}{c}{\quad} & $\widehat{P}$ & \multicolumn{1}{c}{} 
& PIT & SFO & PIT & SFO &\multicolumn{1}{c}{}  & JP & NC & JP & NC 
 & \multicolumn{1}{c}{} & JP & NC & JP & NC \\\hline
\multicolumn{2}{c}{\texttt{MMD}} & $0.52$ & $1$ & $0.98$ & $0.02$ & $0.01 $
& $0.37$ & $0.24$ & $0.47$ & $0.42$ & $0.34$ 
& $0.51$ & $0.15$ & $0.05$  & $0.98$ & $0.97$\\
\multicolumn{2}{c}{\textbf{\texttt{RENAL}}} & \textbf{0.66} & $0.65$ & $0.67$ & $0.58$ & $0.75$
& \textbf{0.57} & $0.57$ & $0.63$ & $0.52$ & $0.54$
& \textbf{0.62} & $0.6$ & $0.37$ & $0.72$ & $0.67$\\\Xhline{2pt}
\end{tabular}}
\end{table*}

\subsection{Synthetic Studies}
\label{exp:syn data}

We investigate the performance of \texttt{RENAL} through synthetic experiments with predefined ground truths and fitted models. More detailed descriptions of the testing procedures are deferred to Appendix~\ref{exp:testing-procedure}. 

\vspace{-0.1in}
\paragraph{Synthetic Data} 
We consider three types of time series data in our synthetic study: ($i$) \emph{Time Series Data} consists of data generated with three models: two stationary ARMA(2,1) and ARMA(2,2) models, and a 
GARCH (1, 1) model, ($ii$) \emph{Temporal Point Processes (TPP) Data} consists of data generated from two TPP models: self-exciting point process (SE) and self-correcting (SC) point process, where the occurrence of one event increases/decreases the likelihood of future events correspondingly, ($iii$) \emph{Spatial-Temporal Point Process (STPP) Data} consists of data generated with two STPP models: the standard diffusion (STD) model which describes symmetric event spread in all directions; the Gaussian diffusion (GAU) \citep{musmeci1992space} model with diffusion shape and orientation varying with location. More details of these models are provided in Appendix~\ref{exp:dataset}.

\vspace{-0.1in}
\paragraph{Result} 
The test accuracies under various scenarios are summarized in Table~\ref{tab:syn_acc}. It shows that \texttt{RENAL} consistently outperforms other methods. It is worthwhile mentioning that when $P^*=\widehat{P}$, the reported accuracy corresponds to Type-I accuracy, whereas when $P^*\neq \widehat{P}$, it corresponds to Type-II accuracy. Additionally, we report the overall accuracy for each of the three types of time series by averaging the results across all scenarios in Table~\ref{tab:syn_acc}. Due to the limited sample size (single sequence), all methods do not exactly meet the pre-specified Type I error rate of $\alpha=0.05$ used to define the rejection region. Despite this, our method, \texttt{RENAL}, consistently achieves one of the highest Type I accuracies and best Type II accuracies with superior performance since:
($i$) It exhibits the highest overall accuracy and balanced performance across all scenarios;
($ii$) It requires few model specifications or assumptions, and ($iii$) It is applicable across all settings while some baseline methods are limited to certain cases.

For \emph{Time Series Data}, \texttt{RENAL} achieves strong overall performance in both Type I and Type II accuracies. In contrast, while some baseline methods can easily distinguish ARMA from GARCH due to their structural differences, they exhibit low accuracy when differentiating between two ARMA models. 
For \emph{TPP} and \emph{STPP Data}, \texttt{RENAL} also effectively capture time-dependency, achieving balanced Type I and II accuracies. However, the baseline methods either have low Type-I accuracy or fail to distinguish different models as reflected by their low type-II accuracy. Furthermore, both \texttt{EWD$-m$} and \texttt{Scott} suffer from imbalance accuracies and highly depend on the number of partitions -- denser partitions lead to sparse transition matrices and thus lower Type II accuracy, while coarser partitions may miss nuanced transition probabilities and reduce the Type I accuracy.  \texttt{RENAL} address these issues with adaptive binning. 

Additionally, the results shown in Table~\ref{tab:syn_acc} support our earlier claim that RENAL is robust to assumption violations. This is because RENAL consistently achieves higher and more stable accuracy than baseline methods, even when the sample size is relatively small (Table~\ref{tab:data} in Appendix~\ref{app:exp}) with only a single training data sequence of moderate length, which could potentially violate Assumption~\ref{assump:sufficient}.

\subsection{Real Data Study}
\label{exp:real-data}

We also assess the effectiveness of \texttt{RENAL} on real data including earthquake and weather data.
Details of the testing procedures are provided in Appendix~\ref{exp:testing-procedure}.

\vspace{-0.1in}
\paragraph{Data Description} We include two datasets for our study on real data: ($i$) \emph{Earthquake Data} contains all earthquake data with magnitudes greater than $3$ in North California \citep{waldhauser2008large} and Japan (from Japan Meteorological Agency (JMA) earthquake catalog) from January $1$, $1993$, to December $31$, $2011$, ($ii$) \emph{Weather Data} includes daily temperature records from Pittsburgh International Airport and San Francisco International Airport starting from January 1st, 1994 \citep{vose2014improved}. 

\vspace{-0.1in}

\paragraph{Results}
In real-world data studies, we will compare our method only to \texttt{MMD}, as most other baselines are inapplicable without knowledge of the ground truth data distribution.

The result is shown in Table \ref{tab:earthquake}.
It can be seen that across all datasets \texttt{RENAL} consistently achieves higher average accuracy, and strikes a better balance between Type I accuracy and Type II accuracy compared with \texttt{MMD}.
This is because \texttt{MMD} fails to model temporal dependencies in data sequences, resulting in extreme cases of either predicting nearly all false positives or all false negatives, which on the other hand highlights our method's efficacy in tackling complex time series testing problems. 
We note that both \texttt{RENAL} and \texttt{MMD} show imbalanced results for the earthquake STPP data, this is likely due to inherent high spatial variation in the NC data.

\section{CONCLUSION}

We proposed a novel framework called the REcurrent NeurAL Goodness-of-Fit test (RENAL), tailored to effectively assess the quality of generative time series models. Our approach transforms the time series into conditionally independent pairs of history embeddings using recurrent neural networks, followed by a chi-square-based goodness-of-fit test to evaluate the dependencies between real and generated data. We provide theoretical justification for this framework and demonstrate that optimizing the discretization of history embeddings enhances its performance. In our numerical studies, RENAL consistently outperforms state-of-the-art methods, achieving higher and more balanced testing accuracy across synthetic and real-world datasets, including time series, point processes, and spatio-temporal point processes.

\bibliographystyle{apalike}
\bibliography{refs.bib}

\newpage
\onecolumn
\appendix

\section{PROOFS}\label{app:proof}

\subsection{Proof of Lemma \ref{lemma:markov}}\label{app:proof-markov}

For any subset $B \subset \mathcal{H}$, the left-hand side of the Markov property becomes
\begin{align*}
& \mathbb{P}(h_{i+1} \in B \mid h_i, h_{i-1}, \ldots, h_1) \\
=& \int_{x \in \mathcal X} \mathbb{P}(h_{i+1} \in B \mid h_i , h_{i-1} , \ldots, h_1,X_{i} = x) f_{X_{i}}(x \mid h_i, h_{i-1}, \ldots, h_1) dx,
\end{align*}
where $f_{X_{i}}(\cdot\mid h_i, h_{i-1}, \ldots, h_1)$ is the conditional density function of $X_{i}$, the $i$-th observation in the time series, conditioned on all past embeddings. Since $h_{i+1}$ is determined by $h_i$ and $X_{i}$, we have:
\begin{equation*}
    \mathbb{P}(h_{i+1} \in B \mid h_i, h_{i-1}, \ldots, h_1, X_{i}=x) = \mathbb{P}(h_{i+1} \in B \mid h_i, X_{i}=x).
\end{equation*}
Moreover, note that the distribution of $X_{i}$ is fully determined by $h_i$ in the decoding process. Thus we have
\begin{equation*}
    f_{X_{i}}(x \mid h_i, h_{i-1}, \ldots, h_1) = f_{X_{i}}(x \mid h_i).
\end{equation*}
Therefore, we simplify the left-hand side expression to
\begin{align*}
    \mathbb P(h_{i+1}\in B |h_i,h_{i-1},\ldots, h_1) 
    & = \int_{x \in \mathcal X} \mathbb P(h_{i+1}\in B|h_i, X_{i}=x) f_{X_{i}}(x|h_i)dx \\
    & = \mathbb P(h_{i+1} \in B |h_i).
\end{align*}

To prove homogeneity, we first note that $\phi$ is homogeneous across different time steps (namely we are using the same embedding network across times), \ie, $h_{i+1} = \phi(x_i, h_i; \theta)$ for all $i$. Therefore, for any $y \in \mathcal H$ and $x \in \mathcal X$,
\begin{equation}\label{app-eq:h_markov}
  \begin{aligned}
     \mathbb P(h_i \in B | h_{i-1} = y, X_{i-1} = x)&  = \mathbb P(\phi(X_{i-1}, h_{i-1}) \in B |h_{i-1} = y, X_{i-1} = x))\\
     & = \mathbb P(\phi(X_{i}, h_{i}) \in B |h_{i} = y, X_{i} = x)) \\
     & = \mathbb P(h_{i+1} \in B | h_{i} = y, X_{i} = x).
\end{aligned}  
\end{equation}
Since the conditional distribution of $X_i$, given $h_i$ is also stationary, as the distribution is determined by the embedding function $g$, we have 
\begin{align*}
    \mathbb P(h_{i+1}\in B|h_i = y) & = \int_{x\in \mathcal X}\mathbb P(h_{i+1} \in B|h_i = y, X_{i} = x)f_{X_{i}|h_i}(x|y)dx\\
    & \overset{(i)}{=} \int_{x\in \mathcal X}\mathbb P(h_i \in B | h_{i-1} = y, X_{i-1} = x) f_{X_{i-1}|h_{i-1}}(x|y)dx\\
    & = \mathbb P(h_{i}\in B|h_{i-1} = y),
\end{align*}
where the equality ({\it i}) used the homogeneous property in \eqref{app-eq:h_markov}.

\subsection{Proof of Lemma \ref{lem:discret-markov}}\label{app:discrete-markov-proof}

Using Markov property of $\{h_i\}_{i\in\mathbb N}$ and the fact that $y_i$ deterministically determined by $h_i$,
    \begin{align*}
    & \mathbb{P}(y_i = a_i \mid y_{i-1} = a_{i-1}, \ldots, y_1 = a_1) \\
    =& \mathbb{P}({h}_i \in \mathcal{H}_{a_i} \mid {h}_{i-1} \in \mathcal{H}_{a_{i-1}}, \ldots, h_1 \in \mathcal{H}_{a_1}) \\
    = & \int_{\eta\in \mathcal{H}_{a_{i-1}}}\mathbb{P}({h}_{i} \in \mathcal{H}_{a_i} \mid {h}_{i-1}=\eta) f_{h_{i-1}}(\eta|{h}_{i-1} \in \mathcal{H}_{a_{i-1}}, \ldots, h_1 \in \mathcal{H}_{a_1}) d\eta\\
    =& \int_{\eta\in \mathcal{H}_{a_{i-1}}}\mathbb{P}({h}_{i} \in \mathcal{H}_{a_i} \mid {h}_{i-1}=\eta) f_{h_{i-1}}(\eta|{h}_{i-1} \in \mathcal{H}_{a_{i-1}}) d\eta\\
    =&\mathbb{P}({h}_{i} \in \mathcal{H}_{a_i} \mid {h}_{i-1} \in \mathcal{H}_{a_{i-1}}) \\   
    =& \mathbb{P}(y_{i} = a_i \mid y_{i-1} = a_{i-1}).
 \end{align*}
Furthermore, when assuming the sequence $\{h_i\}_{i\in \mathbb N}$ is stationary, using the homogeneity of $\{h_i\}_{i\in \mathbb N}$ (Lemma~\ref{lemma:markov}) and the deterministic nature of $y_i$, for arbitrary discrete states $a$ and $\zeta$,
\begin{equation*}
\begin{aligned}
\mathbb{P}\left(y_{i+1} = a|y_i = \zeta\right) 
    & = \mathbb{P}\left(h_{i+1} \in \mathcal{H}_{a}|h_i \in H_{\zeta}\right) \\
    & = \int_{y \in H_{\zeta}}\mathbb P\left(h_{i+1} \in \mathcal{H}_{a}|h_i = y\right)f_{h_{i}}(y|h_{i}\in H_\zeta)dy  \\
    & = \int_{y \in H_{\zeta}}\mathbb P\left(h_i \in \mathcal{H}_{a}|h_{i-1} = y\right)f_{h_{i-1}}(y|h_{i-1}\in H_\zeta)dy \\
    & = \mathbb{P}\left(y_{i} = a|h_{i-1} = \zeta\right).    
\end{aligned}
\end{equation*}

Thus, $\{y_i\}_{i\geq 0}$ has the Markovian and homogeneous property.

\subsection{Proof of Proposition~\ref{prop:markov-transition}}

According to Lemma \ref{lemma:markov}, the embedding sequence $\{h_i\}_{i=1}^n$ forms a homogeneous Markov chain. This implies that the distribution $\mathbb{P}$ can be fully characterized by the transition densities between successive history embeddings.
Thus, for two data sequences $D_0 \sim \mathbb{P}^\star$ and $D_1 \sim \mathbb{\widehat{P}}$, the distributions $\mathbb{P}^\star$ and $\mathbb{\widehat{P}}$ can be uniquely and sufficiently represented by their respective transition mappings ${Q}^\star$ and $\widehat{Q}$. The hypothesis test \eqref{eq:markov-transition-test-objective} that tests the equality of these transition mappings is therefore equivalent to testing whether $\mathbb{P}^\star$ and $\mathbb{\widehat{P}}$ are the same.

\subsection{Proof of Proposition \ref{prop:chi}} \label{app:chi-proof}

The proof basically follows the corresponding proof in \citep{billingsley1961statistical}, with some shortening and elaborations to aid understanding. We will first set up and prove two lemmas, Lemma \ref{lem:lem 3.2} and Lemma \ref{lem:thm 3.1}, which are useful to prove Proposition \ref{prop:chi}. 
\begin{lemma}[Lemma 3.2 in \citep{billingsley1961statistical}] \label{lem:lem 3.2} 
Let $\{x_1, x_2, \ldots, x_n\}$ be a sequence of samples from a first-order stationary discrete Markov chain with $m$ states with a transition count matrix $C_m = (c_{uv})_{u,v = 1}^m$ and transition probability matrix $Q_m = (q_{uv})_{u,v = 1}^m$. Denote $c_u := \sum_{v = 1}^m c_{uv}$ as total counts of state $u$ and $q_{u} := \sum_{v = 1}^m q_{uv}$ as the total probability of state $u$. Then the weak law of large number holds, \ie, for arbitrary state $u$, as $n \to \infty$, 
\begin{equation}
\frac{c_u}{n} \overset{P}{\to} q_{u} , \ \text{as }n\to\infty.
\end{equation}
\end{lemma}
\begin{proof}
Define the random variable $z_s(u) := \mathbbm{1} (x_s = u)$. Then $c_u = \sum_{u = 1}^n z_s(u)$. From stationarity of the chain, it follows that $\mathbb E[z_s(u)] = q_{u}$, so $\mathbb E[c_u/n] = n^{-1}\mathbb E[\sum_{s = 1}^n z_s(u)] =  q_{u}$, 
\begin{multline*}
    Cov(z_s(u), z_{s^\prime}(u)) = 
\mathbb E[(z_s(u) - q_{u})(z_{s^{\prime}}(u) - q_u)] \\
= P(x_s = u, x_{s^\prime}=u) - {q_{u}}^2 = \begin{cases}
         q_{u} - {q_{u}}^2 & \text{if } s^{\prime} = s\\
         q_{u}q_{uu}^{|s^{\prime}-s|} - {q_{u}}^2 & \text{otherwise}
     \end{cases},
\end{multline*}
and 
\begin{align*}
    Var(\frac{c_u}{n}) & = \frac{1}{n^2}\sum_{s = 1}^n \sum_{s^\prime = 1}^n Cov(z_s(u), z_{s^\prime}(u)) \\
    & = \frac{1}{n^2}\left\{n(q_u - q_u^2) + 2 \sum_{s = 1}^n\sum_{s^\prime = s+1}^n \left(q_uq_{uu}^{|s^\prime - s|}-q_u^2\right)\right\}\\
    & \leq \frac{1}{n^2}\left\{n(q_u - q_u^2) + 2 \sum_{s = 1}^n\sum_{s^\prime = s+1}^n \left(q_u^2-q_u^2\right)\right\}\tag{$0 \leq q_{uu}^{\cdot}\leq q_u \leq 1$}\\
    & \leq \frac{1}{n}.
\end{align*}
Therefore $Var(\frac{c_u}{n}) \to 0$ as $n \to \infty$, then by Chebyshev's inequality, for $\epsilon > 0$,
\[
\mathbb P\left(\left|\frac{c_u}{n} - q_u\right|\geq \epsilon\right)\leq \frac{Var(\frac{c_u}{n})}{\epsilon^2} \leq \frac{1}{n\epsilon^2}\to 0, \ \text{as }n \to \infty.
\]
Therefore, as $n \to \infty$, $\frac{c_u}{n} \overset{P}{\to} q_{u}$.
\end{proof}

\begin{lemma}[Theorem 3.1 from \citep{billingsley1961statistical}]\label{lem:thm 3.1} Let $\{x_1, x_2, \ldots, x_n\}$ be a sequence of samples from a first-order stationary discrete Markov chain with $m$ states with a transition count matrix $C_m = (c_{uv})_{u,v = 1}^m$ and transition probability matrix $Q_m = (q_{uv})_{u,v = 1}^m$. Denote $c_u := \sum_{v = 1}^m c_{uv}$ as the total counts at state $u$. Let \[\xi_{uv} = \frac{c_{uv} - c_u q_{uv}}{c_u^{1/2}}.\]
The distribution of the $m^2$-dimensional random vector $\xi = (\xi_{uv})$ converges as $n \to \infty$ to the normal distribution.
\end{lemma}
\begin{proof}
The proof follows the proof for Theorem 3.1 in \citep{billingsley1961statistical}. The process $\{x_n\}$ having a sequence of states $\{a_n\}$ can be viewed as having been generated in the following fashion. Consider an independent collection of random variables $x_1$ and $r_{un}$ (for $u = 1, 2, \ldots, m$ and $n = 1, 2, \ldots$) such that
$P(x_1 = u) = q_{u}$ and $P(r_{un} = v) = q_{uv}$. 

We begin by sampling $x_1$. $x_2$ is defined to be $x_2=r_{x_11}$. If $x_1, x_2, \ldots, x_n$ have been defined, $x_{n+1}$ is defined to be $r_{x_n, w}$, where $w-1$ is the number of $l$ ($1\leq l < n$) such that $x_l = x_n$.

By definition, $\{x_i = a_i, 1 \leq i \leq n\} = \{x_1 = a_1, r_{a_{i-1}w_{i-1}} = a_i, 2 \leq i \leq n\}$ where $w_{i}- 1$ is the number of elements among $\{a_1, \ldots , a_{i-1}\}$ that are equal  to $a_i$. 
Since the variables involved are all distinct and independent, 
\begin{equation*} 
\begin{aligned}
P(x_i = a_i,1 \leq i \leq n) & = P(x_1 = a_1)P(r_{a_1w_1 }= a_2) \cdots P(r_{a_{n-1} w_{n-1}} = a_{n}) \\
& = q_{a_1} q_{a_1a_2}\cdots q_{a_{n-1}a_n}.    
\end{aligned}
\end{equation*}

By process of generation,  $(c_{u1}, c_{u2}, \ldots, c_{um})$ is the frequency count of $\{r_{u1}, \ldots, r_{uc_{u}}\}$. By Lemma \ref{lem:lem 3.2}, as $n \to \infty$, $c_u \to nq_{u}$. 

We can then compare $(c_{u1}^{(0)}, \ldots, c_{um}^{(0)})$ with the frequency count $(c_{u1}^{(1)}, \ldots, c_{um}^{(1)})$ of $\{r_{u1}, \ldots, r_{u[nq_{u}^{(0)}]}\}$.
From the independence of the array $\{r_{un}\}$ and the central limit theorem for multinomial trials, it follows that the $m^2$ random variables $\left(\frac{c_{uv}^{(1)}- [n q_{u}^{(0)}] q_{m_v}^{(0)}}{\sqrt{n q_{uv}^{(0)}}}\right)$ have asymptotic jointly normal distribution. Define $e_w$ as $e_{w} = \mathbbm{1}{(r_{uw} = v)} - q_{uv}^{(0)}$ and put $R_w = e_1 + \ldots + e_w$. It is clear that $e_w$'s are independent and identically distributed with mean 0 and variance $\sigma^2 = q_{uv}^{(0)}(1-q_{uv}^{(0)})$. Now \begin{equation}\frac{c_{uv}^{(1)} - \left[ n q_{u}^{(0)} \right] q_{uv}^{(0)}}{n^\frac{1}{2}} - \frac{c_{uv}^{(0)} - c_u^{(0)} q_{uv}^{(0)}}{n^\frac{1}{2}} = \frac{R_{[nq_{u}^{(0)}]} - R_{c_u^{(0)}}}{n^{\frac{1}{2}}}.\end{equation}
By Lemma \ref{lem:lem 3.2}, $\lim_{n\to\infty} \frac{c_u^{(0)}}{n} = q_{u}^{(0)}$, for arbitrary $\epsilon>0$, $\exists N$ sufficiently large such that $\forall n > N$, $P(|c_u^{(0)} - [nq_{u}^{(0)}]|>n\epsilon^3)<\epsilon$, then 
\begin{align*}
&P\left(\left|\frac{c_{uv}^{(1)} - \left[ n q_{u}^{(0)} \right] q_{uv}^{(0)}}{n^\frac{1}{2}} - \frac{c_{uv}^{(0)} - c_u^{(0)} q_{uv}^{(0)}}{n^\frac{1}{2}}\right|>\epsilon\right)\\
=& P\left(\left|\frac{R_{[n q_{u}^{(0)}]} - R_{c_u^{(0)}}}{n^\frac{1}{2}}\right| > \epsilon\right) \\
\leq & P\left(\left|c_u^{(0)} - [n q_{u}^{(0)}]\right| > n \epsilon^3\right) + P\left(\max_{|w-[n q_{u}^{(0)}]|\leq n \epsilon^3} |R_{[n q_{u}^{(0)}]} - R_w| > \epsilon n^\frac{1}{2}\right)\\
\leq & \epsilon + 2P\left(\max_{1 \leq w \leq n \epsilon^3} |R_w| > \frac{1}{2}\epsilon n^\frac{1}{2}\right)\\
\leq & \epsilon + 2\left(\frac{4}{n\epsilon^2}\right)\left(n \epsilon^3\sigma^2\right)  = (1 + 8 \sigma^2) \epsilon. \tag{by Kolmogorov’s inequality in \citep{feller1958introduction}}
\end{align*}
Since $\epsilon$ can be arbitrarily small, $\frac{c_{uv}^{(1)} - \left[ n q_{u}^{(0)} \right] q_{uv}^{(0)}}{n^{1/2}} - \frac{c_{uv}^{(0)} - c_u^{(0)} q_{uv}^{(0)}}{n^{1/2}}\to 0$, \ie, $\frac{c_{uv}^{(1)} - \left[ n q_{u}^{(0)} \right] q_{uv}^{(0)}}{(nq_{u}^{(0)})^\frac{1}{2}} - \frac{c_{uv}^{(0)} - c_u^{(0)} q_{uv}^{(0)}}{(nq_{u}^{(0)})^\frac{1}{2}}\to 0$ and $\frac{c_{uv}^{(0)} - c_u^{(0)} q_{uv}^{(0)}}{(nq_{u}^{(0)})^\frac{1}{2}}$ has the same limiting distribution of $\frac{c_{uv}^{(1)} - \left[ n q_{u}^{(0)} \right] q_{uv}^{(0)}}{(nq_{u}^{(0)})^\frac{1}{2}}$. By Lemma \ref{lem:lem 3.2}, $\xi_{uv} = \frac{c_{uv}^{(0)} - c_u^{(0)} q_{uv}^{(0)}}{c_u^\frac{1}{2}}$ has the same limiting distribution of $\frac{c_{uv}^{(0)} - c_u^{(0)} q_{uv}^{(0)}}{(nq_{u}^{(0)})^\frac{1}{2}}$. $\xi = (\xi_{uv})$ converges to normal distribution as $n \to \infty$.
\end{proof}

Now recall that we assume two data sequences $D_0$ and $D_1$, and obtain the transition count matrices $C_m^{(k)}=(c_{uv}^{(k)})_{u,v=1}^m$, $k=0,1$, from $D_0$ and $D_1$, respectively. Under the null hypothesis that $D_0$ and $D_1$ follow the same distribution, we have the population transition probability matrices for the learned historical embedding should equal to each other, i.e., $Q^*=\widehat{Q}$, we denote $Q^*=\widehat{Q}=(q_{uv})_{u,v=1}^m$.

 Define $\xi_{uv} = \frac{c_{uv}^{(0)} - c_u^{(0)} q_{uv}}{{c_u^{(0)}}^{\frac{1}{2}}}$. By Lemma \ref{lem:thm 3.1}, $\xi = (\xi_{uv})$ converges to normal distribution as $n\to \infty$. The frequency counts $(c_{u1}^{(k)}, c_{u2}^{(k)},\ldots, c_{um}^{(k)})$ for $k \in \{0, 1\}$ have multinomial distributions with probabilities $(q_{u1}, q_{u2}, \ldots, q_{um})$ and  $\mathbb E(c_{uv}^{(k)})=c_u^{(k)}q_{uv}$. By \citep{pearson1900x}, as $c_u^{(k)}\to\infty$, 
\begin{equation*}
    \sum_{v=1}^m\frac{(c_{uv}^{(k)}-c_u^{(k)}q_{uv})}{c_u^{(k)}q_{uv}} \sim \chi^2_{(m-1)},
\quad \text{and} \quad
    \sum_{u = 1}^m \sum_{v = 1}^m \frac{(c_{uv}^{(k)}-c_u^{(k)}q_{uv})^2}{c_u^{(k)}q_{uv}} \sim \chi^2_{m(m-1)} .
\end{equation*}

Under the null hypothesis, it is intuitively clear that the common estimate is $\widehat{q}_{uv}=\frac{c_{uv}^{(0)} + c_{uv}^{(1)}}{c_u^{(0)} + c_u^{(1)}}$ by \citep{billingsley1961statistical, darwin1959note}. Plugging in the common estimate, we have
\[\sum_{u=1}^m\sum_{v=1}^m\dfrac{(c_{uv}^{(0)}-c_{u}^{(0)}\dfrac{c_{uv}^{(0)}+c_{uv}^{(1)}}{c_u^{(0)}+c_{u}^{(1)}}) ^{2}}{c_{u}^{(0)}\dfrac{c_{uv}^{(0)}+c_{uv}^{(1)}}{c_u^{(0)}+c_{u}}^{(1)}}+ \sum_{u=1}^m\sum_{v=1}^m\dfrac{ (c_{uv}^{(1)}-c_{u}^{(1)}\dfrac{c_{uv}^{(0)}+c_{uv}^{(1)}}{c_{u}^{(0)}+c_{u}^{(1)}}) ^{2}}{c_{u}^{(1)}\dfrac{c_{uv}^{(0)}+c_{uv}^{(1)}}{c_{u}^{(0)}+c_{u}^{(1)}}} \sim \chi^2_{m(m-1)}.\]
The above expression can be simplified as, \begin{align*}
   & \sum_{uv} \dfrac{\left(c_{uv}^{(0)}-c_{u}^{(0)}\dfrac{c_{uv}^{(0)}+c_{uv}^{(1)}}{c_u^{(0)}+c_{u}^{(1)}}\right) ^{2}}{c_{u}^{(0)}\dfrac{c_{uv}^{(0)}+c_{uv}^{(1)}}{c_u^{(0)}+c_{u}^{(1)}}}+ \sum _{uv}\dfrac{ \left(c_{uv}^{(1)}-c_{u}^{(1)}\dfrac{c_{uv}^{(0)}+c_{uv}^{(1)}}{c_{u}^{(0)}+c_{u}^{(1)}}\right) ^{2}}{c_{u}^{(1)}\dfrac{c_{uv}^{(0)}+c_{uv}^{(1)}}{c_{u}^{(0)}+c_{u}^{(1)}}}\\
   =& \sum_{uv}\dfrac{c_u^{(0)}c_{u}^{(1)}}{(c_{uv}^{(0)}+c_{uv}^{(1)})\left( c_u^{(0)}c_{u}^{(1)}\right) ^{2}(c_{u}^{(0)}+c_{u}^{(1)})}   \left( c_{u}^{(0)}+c_{u}^{(1)}\right) \left( c_{u}^{(0)}c_{uv}^{(1)}-c_{u}^{(1)}c_{uv}^{(0)}\right) ^{2}\\
=&\sum _{uv}\dfrac{c_{u}^{(0)}c_{u}^{(1)}}{c_{uv}^{(0)}+c_{uv}^{(1)}}\left( \dfrac{c_{uv}^{(0)}}{c_{u}^{(0)}}-\dfrac{c_{uv}^{(1)}}{c_{u}^{(1)}}\right) ^{2}\\
= & \sum _{uv}\dfrac{c_{u}^{(0)}c_{u}^{(1)}}{c_{uv}^{(0)}+c_{uv}^{(1)}}\left( q_{uv}^{(0)}-q_{uv}^{(1)}\right)^{2}\sim \chi^2_{m(m-1)}.
\end{align*} 

\section{ADDITIONAL TEST STATISTICS DETAILS}\label{app:method}

\paragraph{Smoothing Constraints for Objective Function}\label{app:smoothing}

The complete equation for the second derivative is shown below: 

\begin{equation*}
   \text{S}(Q_m) = -\sqrt{\sum_{u=2}^{m-1} \sum_{v=2}^{m-1} (\nabla^2 q_{uv})^2 }\approx  -\sqrt{\sum_{u=2}^{m-1} \sum_{v=2}^{m-1} \left(q_{u+1,v} + q_{u-1,v} + q_{u,v+1} + q_{u,v-1} - 4q_{uv}\right)^2}.
   \end{equation*}
   This smoothing constraint is chosen as the second derivative, which measures the curvature or the rate of change in the transition probabilities across states. Abrupt transitions (i.e., sudden large changes in transition probabilities between consecutive states) result in large second-order differences. Therefore, minimizing the square of the second-order differences (or equivalently, maximizing the negative square of the second-order differences) encourages smoother transitions. This smoothing regularization results in more stable transitions between states by penalizing abrupt changes, leading to gradual and consistent shifts in transition probabilities.
   
   Incorporating the hyperparameters $\lambda$ to control for the weights of smoothing, the objective is then:
    \begin{align*}
       \text{Objective} 
       & = \lVert \{Q_m^{\star}\} - \widehat{Q}_m\} \rVert_F + \lambda  \left( \text{S}(Q_m^\star) +  \text{S}(\widehat{Q}_m)\right)\\
       & = \sqrt{\sum_{u,v}\left(q_{uv}^{\star} - \widehat{q}_{uv}\right)^2} - \lambda\left( \sqrt{\sum_{u=2}^{m-1} \sum_{v=2}^{m-1} (\nabla^2 q_{uv}^{\star})^2} + \sqrt{\sum_{u=2}^{m-1} \sum_{v=2}^{m-1} (\nabla^2 \widehat{q}_{uv})^2}\right).
    \end{align*}

\section{EXPERIMENTS DETAILS}\label{app:exp}
The following sections provide more details of the synthetic and real-world experiments. The code is available in \href{https://github.com/aoranzhangmia/Neural-GoF-Time}{https://github.com/aoranzhangmia/Neural-GoF-Time}.

\subsection{Baseline Details}\label{exp:baseline}
For \texttt{MMD}, \texttt{KSD}, and \texttt{Stein}, input sequences are split into $50$ subsequences and the radial basis function kernel is used with bandwidth set as the median of pairwise squared Euclidean distances among data points. \texttt{KSD} utilizes the trapezoidal rule with $101$ points for integration. For \texttt{Stein} test, the bootstrap flip probability is set at $0.01$. 

For discretization baselines, we construct \texttt{EWD$-m$} and \texttt{Scott} discretization for a history embedding sequence $\{h_i\} \subset \mathcal H \subseteq \mathbb R^{p}$ as follows: $(i)$ \texttt{EWD$-m$}: for each dimension $p^\prime = 1, 2, \ldots, p$ of the latent space $\mathcal H \subseteq \mathbb R^p$, we partition the range of $\{h_i\}_{p^\prime}$ (the $p^\prime$-th dimension of $\{h_i\}$ into $m$ equally sized intervals (bins). $(ii)$ \texttt{Scott}: For each dimension $p^\prime$, the number of bins ($m^{p^\prime}$) for the equal-sized binning is determined by Scott's rule \citep{scott1979optimal}:
\[m^{p^\prime} = \left\lceil \frac{\max\left(\{h_i\}_{p^\prime}\right) - \min \left(\{h_i\}_{p^\prime}\right)}{3.5\cdot\sigma_{\{h_i\}_{p^\prime}}\cdot n^{-1/3}} \right\rceil, \]
where $\sigma_{\{h_i\}_{p^\prime}}$ is the standard deviation of the sequence $\{h_i\}_{p^\prime}$ and $n$ is the length of the $\{h_i\}$.

\subsection{Time Series Generative Model Details}\label{exp:nn-details}
The detailed setting for the time series generative models used in the experiment can be referred to Table~\ref{tab:nn-param}. Here $\ell$ denotes learning rate, $e$ denotes epoches, $N_b$ represents the maximum number of bins per dimension, and $\lambda$ sets the corresponding smoothing constraints. The maximum number of discretized states is 100, and ensure that transition count matrices have at least 15\% non-zero entries to prevent overly sparse transition probability matrices. Note that as the dimension of laten space increase, the proportion of non-zero entries and $\lambda$ should be increased accordingly to avoid too sparse transition probability matrices.
\begin{table}[!t]
\centering
\caption{Neural Network Specifications and Discretization Settings.
}
\label{tab:nn-param}
\resizebox{0.85\linewidth}{!}{
\begin{tabular}{c ccccc c c}
\Xhline{1pt}
\multirow{2}{*}{Type of Models} & \multicolumn{5}{c}{Neural Network } & \multicolumn{2}{c}{Discretization} \\ \cline{2-8} 
& Network Type & Hidden Dimension & $\ell$ & Optimizer & $e$ & $N_b$ & $\lambda$ \\ \hline
Time Series & LSTM & 6 & 0.001 & Adam & 100 & 6 & $0.1$\\ 
TPP & NeuralHawkes  & 4 & 0.00025 & SGD & 150 & 20 & $0.08$ \\ 
STPP & NHPP & 4 & 0.0005 & SGD & 200 & 20 & $0.08$       \\\Xhline{1pt}
\end{tabular}}
\end{table}

\subsection{Datasets Details} \label{exp:dataset}
The general setting is presented in Table~\ref{tab:data}. \textbf{Type} denotes the model type if the dataset is synthetic and denotes the kind of data if the dataset is obtained from real observation. \textbf{Range} denotes the time range (one-dimensional) if the data is time series data, and time $\times$ longitude $\times$ latitude space (three dimensions) if it is spatial-temporal data. \textbf{Average Length} denotes the average number of data points in a single trajectory.

The model specification for synthetic data and specific information of real data are listed below:
\begin{table*}[!t]
\centering
\caption{ Synthetic and real dataset description. }
\label{tab:data}
\resizebox{\textwidth}{!}{
\begin{tabular}{c c c c  c c  c c  c c c}
\Xhline{2pt}
& \multicolumn{3}{c}{\textbf{Time Series}} & \multicolumn{2}{c}{\textbf{TPP}} & \multicolumn{2}{c}{\textbf{STPP}} & \multicolumn{2}{c}{\textbf{Real Data}} \\ \cline{2-11}
\textbf{Types}  & ARMA$_1$ & ARMA$_2$ & GARCH & SE & SC & STD & GAU & Earthquake & Weather \\ 
\textbf{Range}  & [0,50) & [0,50) & [0,50) & [0,100) & [0,100) & \multicolumn{2}{c}{Time: $[0,1)$, Space: $[0,1)\times[0,1)$} & Time: $[0,1)$, Space: $[0,1)\times[0,1)$ & [0,365) \\ 
\textbf{Average Length} & 500 & 500 & 500 & 500 & 500 & 500 & 500 & 400 & 365 \\
\textbf{\# of Trajectories} & 500 & 500 & 500 & 500 & 500 & 200 & 200 & 200 & 60\\
\Xhline{2pt}
\end{tabular}}
\label{tab:dataset_summary}
\end{table*}
\paragraph{Time Series Data}  For an ARMA(p,q) model with autoregressive (AR) coefficients $\phi_{arma} = (\phi_1, \ldots, \phi_p)$, moving average (MA) coefficients $\theta_{arma} = (\theta_1, \ldots, \theta_q)$, and variance of the white noice ($\sigma^2$), the model is given by:
\[x_i = \phi_1x_{i-1} + \phi_2x_{i-2}+\cdots + \phi_{p}x_{i-p} + \epsilon_i + \theta_1\epsilon_{i-1} + \cdots + \theta_q\epsilon_{i-q},\]
where $x_i$ is the value of the time series at index $i$ and $\epsilon_i$ is the white noise term at index $i$ with $i \sim N(0, \sigma^2)$.

ARMA models we experiment are specified as:
\begin{enumerate}
    \item $ARMA_1$: ARMA(2,1) with $\phi_{arma} =[0.5, 0.4]$, $\theta_{arma} = 0.65$ and $\sigma = 1$.
    \item $ARMA_2$: ARMA (2,2) with $\phi_{arma} =[0.5, -0.4], \theta_{arma} =[0.3, -0.2]$ and $\sigma = 1$.
\end{enumerate} 

Note that the neural ODE framework can be used to approximately model the ARMA time series. AR(1) model can be exactly modeled using our framework by letting $h_i = \phi(h_{i-1},x_{i-1})=x_{i-1}=h_{i-1}$ and $x_i=\alpha h_i+\epsilon_i$. For AR(2), $h_{i}=(x_{i-1},x_{i-2})$, thus $h_i=\phi(h_{i-1},x_{i-1})$ and $x_i=\phi^\top h_i + \epsilon$, \ie, the $g$ function is the linear function $\phi^\top h_i$. For ARMA models, the MA part involves past errors that make it complicated, but the neural ODE can still approximate how these errors propagate over time, thus implicitly capturing the time dependencies within the process.

The GARCH($p,q$) model measuring the volatility over a time range is defined with the following parameters: the mean of the time series is $\mu$, the constant term representing the baseline level of volatility is $\omega$, the ARCH terms capturing the impact of past shocks are $\alpha_{garch} = (\alpha_1, \ldots, \alpha_p)$, the GARCH terms capturing the persistence of volatility are $\beta_{garch} = (\beta_1, \ldots, \beta_q)$, the leverage terms capturing the asymmetry for negative shocks are $\gamma_{garch} = (\gamma_1, \ldots, \gamma_p)$, and the innovation term (shock) at time $i$ is $\eta_i$.
Given the parameters, GARCH$(p,q)$ is defined with mean and variance equations:
\begin{align*}
& \text{Mean Equation:}\quad x_i = \mu + \eta_i, \\
& \text{Variance Equation:}\quad\sigma_i^2 = \omega + \sum_{j=1}^{p} \alpha_j \eta_{i-j}^2 + \sum_{j=1}^{q} \beta_j \sigma_{i-j}^2 + \sum_{j=1}^{p} \gamma_l \eta_{i-j} \mathbbm{1}(\eta_{i-j} < 0),
\end{align*}
where $\eta_i = \epsilon_i \sigma_i$ represents the innovation with $\epsilon_i \sim N(0,1)$ and $\sigma_i$ being the conditional standard deviation at index $i$. In the experiment, the model GARCH(1,1) is defined with $\mu = 0.03, \omega = 0.04, \alpha_{garch} = 0.04, \gamma_{garch} = 0.02, \beta_{garch} = 0.9$, and $\eta = 1.0$.

\paragraph{TPP Data}\label{data:tpp} Sequence are  simulated using Ogata's thinning algorithm \citep{ogata1981lewis}. The models are set as follows:
\begin{enumerate}
    \item \emph{Self-Exciting process (SE)}: The occurrence of past events triggers future events. The conditional intensity function is given by $\lambda(t|\mathcal{H}_i) = \mu + \alpha\sum_{i:t_i < t}\exp\{-\beta(t-t_i)\}$, where $\mu = 1, \alpha = 1$, $\beta = 1.25$ and $\bar \lambda = 1e+2$.
    \item \emph{Self-Correcting Process (SC)}: The occurrence of events reduces the likelihood of future ones. The conditional intensity function is $\lambda(t) = \exp\left(\mu + \alpha t - \sum_{i:t_i < t} \beta \right)$, where $\mu = 2.5, \alpha = 0.05$, $\beta = 0.25$, and $\bar \lambda = 1e+2$.
\end{enumerate}
\paragraph{STPP Data} The standard diffusion kernel function is defined as:
\begin{equation*}
    h(t, {s}, h_i, {h_s}) = C\frac{\exp(-\beta\Delta t)}{2\pi\sigma_x\sigma_y\Delta t}\exp\left\{-\frac{1}{2\Delta t}\left(\frac{{(\Delta x)}^2}{\sigma^2_ x} + \frac{(\Delta y)^2}{\sigma_y^2}\right)\right\},
\end{equation*}
and the Gaussian diffusion kernel function is defined as:
\begin{multline*}
    h(t, {s}, h_i, {h_s}) =   
    \frac{C\exp(-\beta \Delta t)}{2 \pi \sigma_x \sigma_y \Delta t \sqrt{1 - \rho^2}}\times \\
    \exp\left(-\frac{ 
   (\Delta x - \mu_x)^2/\sigma_x^2 + (\Delta y - \mu_y)^2/\sigma_y^2 - 
    \left(2 \rho (x - \mu_x) ( \Delta y - \mu_y)\right)/\left(\sigma_x \sigma_y \right)}{2 \Delta t (1 - \rho^2)}  \right),
\end{multline*}
where $h_i$ is a historical time, ${h}_s$ is a historical location, $\Delta t = t - h_i$, $\Delta s = s - h_s$, $\Delta x = \Delta s_x$, and $\Delta y = \Delta s_y$. Models are defined with the following parameters:
\begin{enumerate}
    \item $STPP_{std}$: $\mu = 1$, $\sigma_x = \sigma_y = 0.5, \beta = 0.25, C = 1$, $\bar\lambda = 1e+4$.
    \item $STPP_{gau}$: $\mu = 2.5$, $\mu_x= \mu_y=0.1, \sigma_x=\sigma_y=1, \rho=0, \beta=1, C=1$, $\bar\lambda = 1e+4$.
\end{enumerate}

\begin{figure}[!t]
    \centering
        \includegraphics[width=0.6\linewidth]{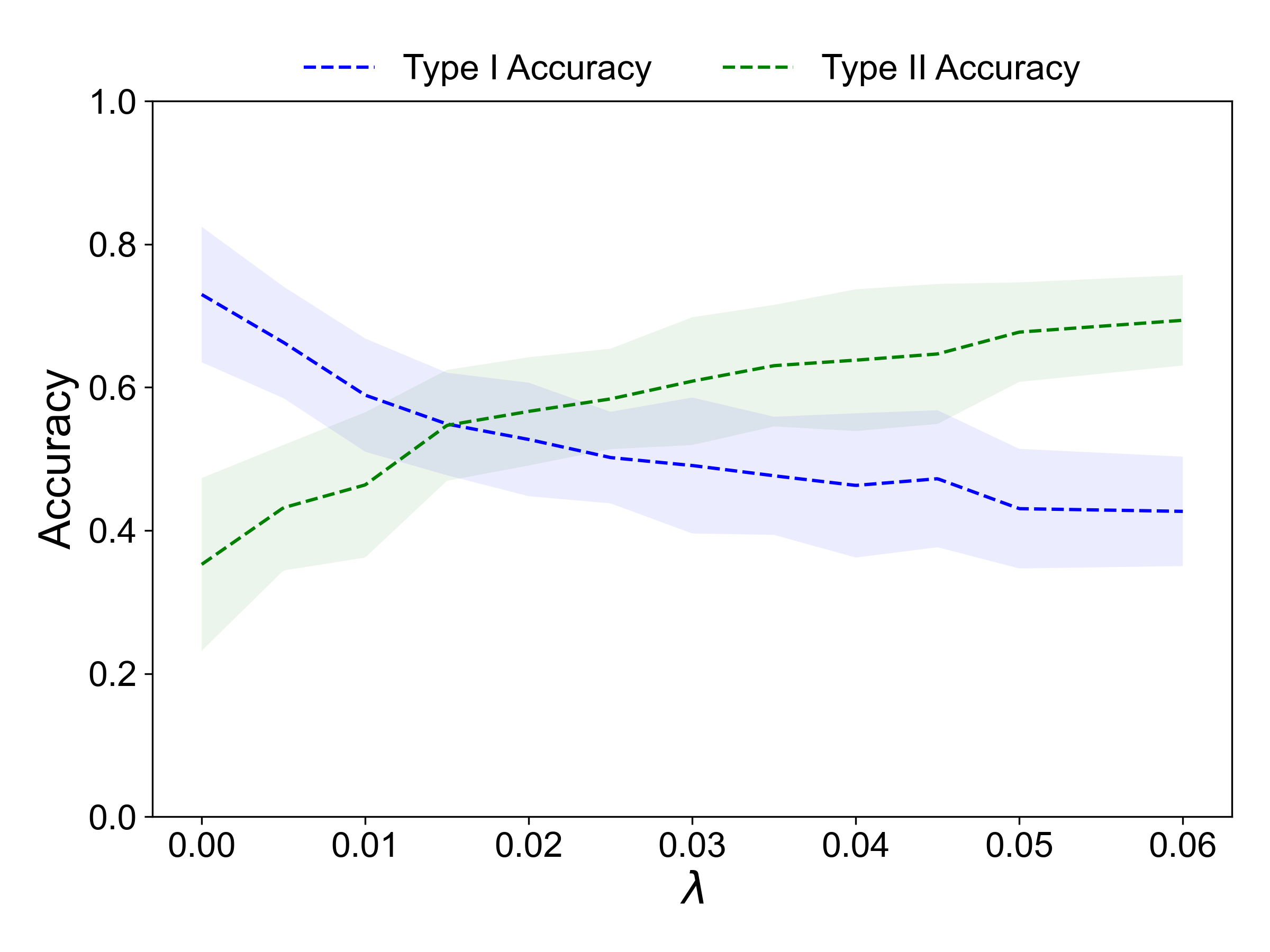}
    \caption{
    Ablation study of hyper-parameter $\lambda$ selection. }
    \label{fig:hyper}
\end{figure}

\subsection{Testing Procedure Details}\label{exp:testing-procedure}
\paragraph{Synthetic Data} For each dataset, we perform 500 iterations of testing between randomly selected pairs of a designated ground-truth model and a fitted model, where $P^\star$ represents the ground-truth model and $\widehat{P}$ represents the fitted model.
\paragraph{Real Data} 
 Observations are clipped randomly and normalized for each iteration of the experiment with an average length $400$. We evaluate the performance of goodness-of-fit test statistics for four cases: \textit{(i)} North California earthquake data sequences as the ground truth and Japan California earthquake data as the fitted model; {\it (ii)} the reverse configuration; {\it (iii)} North California data as both observe and fitted model; {\it (iv)} Japan data for both observe and fitted model. 200 iterations are performed for each scenario.
Similar to experiments on Earthquake dataset, we assess the performance of proposed test in each scenario for 200 iterations.

\paragraph{Ablation Study} This study examines the impact of hyperparameter choices on the testing performance of the NR statistic. We focus on the smoothing constraints $\lambda$ for discretization. We set the ground-truth model as SE and the fitted model as SC using the same experimental setup as the TPP synthetic experiment. 200 iterations are performed for each $e$ and $\lambda$. Now the dimension of the hidden states is $2$ and $e$ is 150.

\emph{Smoothing Constraints $\lambda$}: As $\lambda$ increases, we observe a general increase in Type II accuracies and a decrease in Type I accuracies. Type II accuracies begin around 0.35, increase to 0.5 by $\lambda = 0.01$, and continue to increase until $\lambda = 0.015$, where the trend of increase become slower until it stabilizes beyond $\lambda = 0.05$. Type II accuracies start near 0.75 for small $\lambda$, drop below 0.6 by $\lambda = 0.075$ and continue the decreasing trend until it becomes stable beyond $\lambda = 0.06$. Type I and II accuracies intersect at $\lambda = 0.015$, achieving the best balance between Type I and Type II Accuracy.

\end{document}